\documentclass[12pt]{article}
\usepackage{fullpage}
\usepackage{hyperref}
\usepackage{amsthm, amsmath, amsfonts}
\usepackage{bbm}
\usepackage{graphicx}
\usepackage{caption}
\usepackage{subcaption}
\usepackage{mathtools}
\usepackage{verbatim}
\usepackage{epstopdf}

\usepackage{natbib}
\newtheorem{prop}{Proposition}
\newtheorem{theor}{Theorem}

\newtheorem{lemma}{Lemma}

\newcommand{\footremember}[2]{%
    \footnote{#2}
    \newcounter{#1}
    \setcounter{#1}{\value{footnote}}%
}

\DeclareMathOperator*{\esssup}{ess\,sup}

\begin{document}
\title{Error bounds for approximations with deep ReLU networks}
\author{Dmitry Yarotsky\footremember{skoltech}{Skolkovo Institute of Science and Technology, Skolkovo Innovation Center, Building 3, Moscow  143026
Russia}\footremember{iitp}{Institute for Information Transmission Problems, Bolshoy Karetny per. 19, build.1, Moscow 127051, Russia}\\
\texttt{d.yarotsky@skoltech.ru}
}
\maketitle

\begin{abstract} We study expressive power of shallow and deep neural networks with piece-wise linear activation functions. We establish new rigorous upper and lower bounds for the network complexity in the setting of approximations in Sobolev spaces. In particular, we prove that deep ReLU networks more efficiently approximate smooth functions than shallow networks. In the case of approximations of 1D Lipschitz functions we describe adaptive depth-6 network architectures more efficient than the standard shallow architecture. 
\end{abstract}

\section{Introduction}
Recently, multiple successful applications of deep neural networks to pattern recognition problems (\cite{schmidhuber2015deep, lecun2015deep}) have revived active interest in theoretical properties of such networks, in particular their expressive power. It has been argued that deep networks may be more expressive than shallow ones of comparable size (see, e.g., \cite{delalleau2011shallow, raghu2016expressive, montufar2014number, bianchini2014complexity, telgarsky2015representation}). In contrast to a shallow network, a deep one can be viewed as a long sequence of non-commutative transformations, which is a natural setting for high expressiveness  (cf. the well-known Solovay-Kitaev theorem on fast approximation of arbitrary quantum operations by sequences of non-commutative gates, see \cite{Kitaev:2002,dawson2006solovay}). 

There are various ways to characterize expressive power of networks. Delalleau and Bengio \citeyear{delalleau2011shallow} consider sum-product networks and prove for certain classes of polynomials  that they are much more easily represented by deep networks than by shallow networks. Montufar et al. \citeyear{montufar2014number} estimate the number of linear regions in the network's landscape. Bianchini and Scarselli  \citeyear{bianchini2014complexity} give bounds for Betti numbers characterizing topological properties of functions represented by networks. Telgarsky \citeyear{telgarsky2015representation, telgarsky2016benefits} provides specific examples of classification problems where deep networks are provably more efficient than shallow ones. 

In the context of classification problems, a general and standard approach to characterizing expressiveness is based on the notion of the Vapnik-Chervonenkis dimension (\cite{vapnik2015uniform}). There exist several bounds for VC-dimension of deep networks with piece-wise polynomial activation functions that go back to geometric techniques of Goldberg and Jerrum \citeyear{goldberg1995bounding} and earlier results of Warren \citeyear{warren1968lower}; see \cite{bartlett1998almost, sakurai1999tight} and the book \cite{anthony2009neural}. There is a related concept, the fat-shattering dimension, for real-valued approximation problems (\cite{kearns1990efficient, anthony2009neural}). 

A very general approach to expressiveness in the context of approximation is the method of nonlinear widths by DeVore et al. \citeyear{devore1989optimal} that concerns approximation of a family of functions under assumption of a continuous dependence of the model on the approximated function. 

In this paper we examine the problem of shallow-vs-deep expressiveness from the perspective of approximation theory and general spaces of functions having derivatives up to certain order (Sobolev-type spaces). In this framework, the problem of expressiveness is very well studied in the case of shallow networks with a single hidden layer, where it is known, in particular, that to approximate a $C^n$-function on a $d$-dimensional set with infinitesimal error $\epsilon$ one needs a network of size about $\epsilon^{-d/n}$, assuming a smooth activation function (see, e.g., \cite{mhaskar1996neural, pinkus1999review} for a number of related rigorous upper and lower bounds and further qualifications of this result). Much less seems to be known  about deep networks in this setting, though Mhaskar et al. \citeyear{mhaskar2016learning, mhaskar2016deep} have recently  introduced functional spaces constructed using deep dependency graphs and obtained expressiveness bounds for related deep networks.     

We will focus our attention on networks with the ReLU activation function $\sigma(x)=\max(0,x)$, which, despite its utter simplicity, seems to be the most popular choice in practical applications \cite{lecun2015deep}. We will consider $L^\infty$-error of approximation of functions belonging to the Sobolev spaces $\mathcal W^{n,\infty}([0,1]^d)$ (without any assumptions of hierarchical structure). We will often consider families of approximations, as the approximated function runs over the unit ball $F_{d,n}$ in $\mathcal W^{n,\infty}([0,1]^d)$. In such cases we will distinguish scenarios of fixed and adaptive network architectures. Our goal is to obtain lower and upper bounds on the expressiveness of deep and shallow networks in different scenarios. We measure complexity of networks in a conventional way, by counting the number of their weights and computation units (cf. \cite{anthony2009neural}).     

The main body of the paper consists of Sections \ref{sec:setting}, \ref{sec:benefit} and \ref{sec:futility}.

In Section \ref{sec:setting} we describe our ReLU network model and show that the ReLU function is replaceable by any other continuous piece-wise linear activation function, up to constant factors in complexity asymptotics (Proposition~\ref{th:piecewise}).

In Section \ref{sec:benefit} we establish several upper bounds on the complexity of approximating by ReLU networks, in particular showing that deep networks are quite efficient for approximating smooth functions. Specifically:
\begin{itemize}
\item In Subsection \ref{sec:squaring} we show that the function $f(x)=x^2$ can be $\epsilon$-approximated by a network of depth and complexity $O(\ln (1/\epsilon))$ (Proposition \ref{th:x2}). This also leads to similar upper bounds on the depth and complexity that are sufficient to implement an approximate multiplication in a ReLU network (Proposition \ref{th:fg}). 
\item In Subsection \ref{sec:gensmooth} we describe a ReLU network architecture of depth $O(\ln (1/\epsilon))$ and complexity $O(\epsilon^{-d/n}\ln (1/\epsilon))$ that is capable of approximating with error $\epsilon$ any function from $F_{d,n}$ (Theorem \ref{th:gensmooth}). 
\item In Subsection \ref{sec:faster} we show that, even with fixed-depth networks, one can further decrease the approximation complexity if the network architecture is allowed to depend on the approximated function. Specifically, we prove that one can $\epsilon$-approximate a given Lipschitz function on the segment $[0,1]$ by a depth-6 ReLU network with $O(\frac{1}{\epsilon \ln (1/\epsilon)})$ connections and activation units (Theorem \ref{th:faster}). This upper bound is of interest since it lies below the lower bound provided by the method of nonlinear widths under assumption of continuous model selection (see Subsection \ref{sec:dwidth}).
\end{itemize}
In Section \ref{sec:futility} we obtain several lower bounds on the complexity of approximation by deep and shallow ReLU networks, using different approaches and assumptions.
\begin{itemize}
\item In Subsection \ref{sec:dwidth} we recall the general lower bound provided by the method of continuous nonlinear widths. This method assumes that parameters of the approximation continuously depend on the approximated function, but does not assume anything about how the approximation depends on its parameters. In this setup, at least $\sim \epsilon^{-d/n}$ connections and weights are required for an $\epsilon$-approximation on $F_{d,n}$ (Theorem \ref{th:dwidth}). As already mentioned, for $d=n=1$ this lower bound is above the upper bound provided by Theorem~\ref{th:faster}.
\item In Subsection \ref{sec:vc} we consider the setup where the same network architecture is used to approximate all functions in $F_{d,n}$, but the weights are not assumed to continuously depend on the function. In this case, application of existing results on VC-dimension of deep piece-wise polynomial networks yields a $\sim \epsilon^{d/(2n)}$ lower bound in general and a $\sim \epsilon^{-d/n}\ln^{-2p-1} (1/\epsilon)$ lower bound if the network depth grows as $O(\ln^p (1/\epsilon))$  (Theorem~\ref{th:vc}).  
\item In Subsection \ref{sec:lb_funcdep} we consider an individual approximation, without any assumptions regarding it as an element of a family as in Subsections \ref{sec:dwidth} and \ref{sec:vc}.  We prove that for any $d,n$ there exists a function in $\mathcal W^{n,\infty}([0,1]^d)$ such that its approximation complexity is not $o(\epsilon^{-d/(9n)})$ as $\epsilon\to 0$ (Theorem \ref{th:lb_funcdep}).
\item In Subsection \ref{sec:slow} we prove that $\epsilon$-approximation of any nonlinear $C^2$-function by a network of fixed depth $L$ requires at least $\sim \epsilon^{-1/(2(L-2))}$ computation units (Theorem \ref{th:slow}). By comparison with Theorem \ref{th:gensmooth}, this shows that for sufficiently smooth functions approximation by fixed-depth ReLU networks is less efficient than by unbounded-depth networks. 
\end{itemize}
In Section \ref{sec:discus} we discuss the obtained bounds and summarize their implications, in particular comparing deep vs. shallow networks and fixed vs. adaptive architectures.

The arXiv preprint 
of the first version of the present work appeared almost simultaneously with the work of Liang and Srikant \cite{liang2016why} containing results partly overlapping with our results in Subsections \ref{sec:squaring},\ref{sec:gensmooth} and \ref{sec:slow}. Liang and Srikant consider networks equipped with both ReLU and threshold activation functions. They prove a logarithmic upper bound for the complexity of approximating the function $f(x)=x^2$, which is analogous to our Proposition \ref{th:x2}. 
Then, they extend this upper bound to polynomials and smooth functions. In contrast to our treatment of generic smooth functions based on standard Sobolev spaces, they impose more complex assumptions on the function (including, in particular, how many derivatives it has) that depend on the required approximation accuracy $\epsilon$. As a consequence, they obtain strong $O(\ln^c(1/\epsilon))$ complexity bounds rather different from our bound in Theorem \ref{th:gensmooth} (in fact, our lower bound proved in Theorem \ref{th:lb_funcdep} rules out, in general, such strong upper bounds for functions having only finitely many derivatives). Also, Liang and Srikant prove a lower bound for the complexity of approximating convex functions by shallow networks. Our version of this result, given in Subsection \ref{sec:slow}, is different in that we assume smoothness and nonlinearity instead of global convexity.

\section{The ReLU network model}\label{sec:setting}

\begin{figure}
\begin{center}
\includegraphics[width=0.4\textwidth, trim= 15mm 10mm 10mm 10mm,clip]{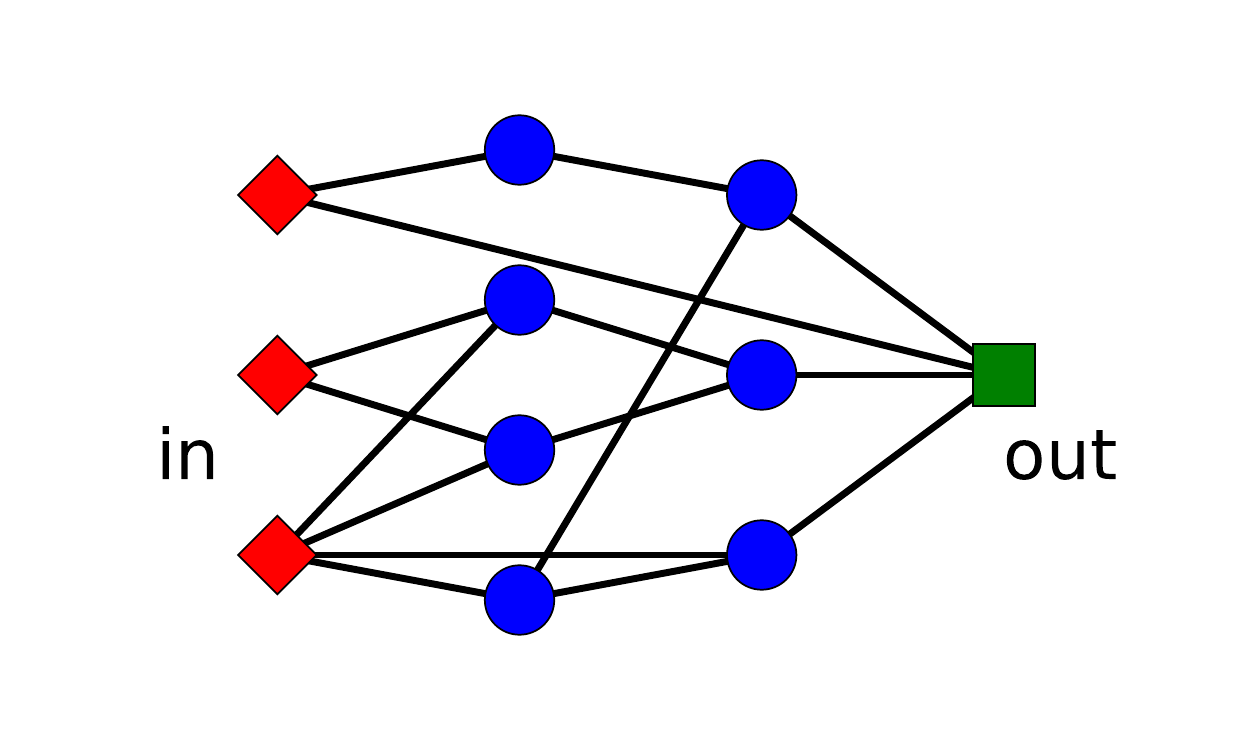}
\end{center}
\caption{A feedforward neural network having 3 input units (diamonds), 1 output unit (square), and 7 computation units with nonlinear activation (circles). The network has 4 layers and $16+8=24$ weights.}\label{nn}
\end{figure}

Throughout the paper, we consider feedforward neural networks with the ReLU (Rectified Linear Unit) activation function $$\sigma(x)=\max(0,x).$$
The network consists of several input units, one output unit, and a number of ``hidden'' computation units. Each hidden unit performs an operation of the form 
\begin{equation}\label{eq:relu}
y=\sigma\Big(\sum_{k=1}^N w_kx_k+b\Big)
\end{equation} 
with some weights (adjustable parameters) $(w_k)_{k=1}^N$ and $b$ depending on the unit. The output unit is also a computation unit, but without the nonlinearity, i.e., it computes $y=\sum_{k=1}^N w_kx_k+b$. The units are grouped in layers, and  the inputs $(x_k)_{k=1}^N$ of a computation unit in a certain layer are outputs of some units belonging to any of the preceding layers (see Fig. \ref{nn}). Note that we allow connections between units in non-neighboring layers. Occasionally, when this cannot cause confusion, we may denote the network and the function it implements by the same symbol.

The depth of the network, the number of units and the total number of weights are standard measures of network complexity (\cite{anthony2009neural}). We will use these measures throughout the paper. The number of weights is, clearly, the sum of the total number of connections and the number of computation units. We identify the depth with the number of layers (in particular, the most common  type of neural networks -- shallow networks having a single hidden layer -- are depth-3 networks according to this convention). 

We finish this subsection with a proposition showing that, given our complexity measures, using the ReLU activation function is not much different from using any other piece-wise linear activation function with finitely many breakpoints:  one can replace one network by an equivalent one but having another activation function while only increasing the number of units and weights by constant factors. This justifies our restricted attention to the ReLU networks (which could otherwise have been perceived as an excessively particular example of networks).  

\begin{prop}\label{th:piecewise} Let $\rho:\mathbb R\to\mathbb R$ be any continuous piece-wise linear function with $M$ breakpoints, where $1\le M<\infty$. 

\begin{itemize}
\item [a)] Let $\xi$ be a network with the activation function $\rho$, having depth $L$, $W$ weights and $U$ computation units. Then there exists a ReLU network $\eta$ that has depth $L$, not more than $(M+1)^2W$ weights and not more than $(M+1)U$ units, and that computes the same function as $\xi$. 
\item [b)] Conversely, let $\eta$ be a ReLU network of depth $L$ with $W$ weights and $U$ computation units. Let $\mathcal D$ be a bounded subset of $\mathbb R^n$, where $n$ is the input dimension of $\eta$. Then there exists a network with the activation function $\rho$ that has depth $L$, $4W$ weights and $2U$ units, and that computes the same function as $\eta$ on the set $\mathcal D$. 
\end{itemize}
\end{prop}
\begin{proof}
a) Let $a_1<\ldots<a_M$ be the breakpoints of $\rho$, i.e., the points where its derivative is discontinuous: $\rho'(a_k+)\ne \rho'(a_k-)$. We can then express $\rho$ via the ReLU function $\sigma$, as a linear combination
$$\rho(x) = c_0\sigma(a_1-x)+\sum_{m=1}^Mc_m\sigma(x-a_m)+h$$
with appropriately chosen coefficients $(c_m)_{m=0}^M$ and $h$. It follows that computation performed by a single $\rho$-unit, $$x_1,\ldots,x_N \mapsto \rho\Big(\sum_{k=1}^N w_kx_k+b\Big),$$
can be equivalently represented by a linear combination of a constant function and computations of $M+1$ $\sigma$-units,
$$x_1,\ldots,x_N \mapsto 
\begin{cases}
\sigma\Big(\sum_{k=1}^N w_kx_k+b-a_m\Big), & m=1,\ldots,M,\\
\sigma\Big(a_1-b-\sum_{k=1}^N w_kx_k), & m=0
\end{cases}$$
(here $m$ is the index of a $\rho$-unit).
We can then replace one-by-one all the $\rho$-units in the network $\xi$ by $\sigma$-units, without changing the output of the network. Obviously, these replacements do not change the network depth. Since each hidden unit gets replaced by $M+1$ new units, the number of units in the new network is not greater than $M+1$ times their number in the original network. Note also that the number of connections in the network is multiplied, at most, by $(M+1)^2$. Indeed, each unit replacement entails replacing each of the incoming and outgoing connections of this unit by $M+1$ new connections, and each connection is replaced twice: as an incoming and as an outgoing one. These considerations imply the claimed complexity bounds for the resulting $\sigma$-network $\eta$.     

b) Let $a$ be any breakpoint of $\rho$, so that $\rho'(a+)\ne \rho'(a-)$. 
Let $r_0$ be the distance separating $a$ from the nearest other breakpoint, so that $\rho$ is linear on $[a,a+r_0]$ and on $[a-r_0,a]$ (if $\rho$ has only one node, any $r_0>0$ will do). Then, for any $r>0$, we can express the ReLU function $\sigma$ via $\rho$ in the $r$-neighborhood of 0:
$$\sigma(x) = \frac{\rho\big(a+\frac{r_0}{2r}x\big)-\rho\big(a-\frac{r_0}{2}+\frac{r_0}{2r}x\big)-\rho(a)+\rho\big(a-\frac{r_0}{2}\big)}{\big(\rho'(a+)-\rho'(a-)\big)\frac{r_0}{2r}},\qquad x\in[-r, r].$$ 
It follows that a computation performed by a single $\sigma$-unit,
$$x_1,\ldots,x_N \mapsto \sigma\Big(\sum_{k=1}^N w_kx_k+b\Big),$$
can be equivalently represented by a linear combination of a constant function and two $\rho$-units,
$$x_1,\ldots,x_N \mapsto 
\begin{cases}
\rho\Big(a+\frac{r_0}{2r}b+\frac{r_0}{2r}\sum_{k=1}^N w_kx_k\Big), \\
\rho\Big(a-\frac{r_0}{2}+\frac{r_0}{2r}b+\frac{r_0}{2r}\sum_{k=1}^N w_kx_k\Big), 
\end{cases}$$
provided the condition \begin{equation}\label{eq:plf}\sum_{k=1}^N w_kx_k+b\in [-r,r]\end{equation} holds. Since $\mathcal D$ is a bounded set, we can choose $r$ at each unit of the initial network $\eta$ sufficiently large so as to satisfy condition \eqref{eq:plf} for all network inputs from $\mathcal D$. Then, like in a), we replace each $\sigma$-unit with two $\rho$-units, which produces the desired $\rho$-network. 
\end{proof}

\section{Upper bounds
}\label{sec:benefit}
Throughout the paper, we will be interested in approximating functions $f:[0,1]^d\to\mathbb R$ by ReLU networks. Given a function $f:[0,1]^d\to\mathbb R$ and its approximation $\widetilde f$, by the \emph{approximation error} we will always mean the uniform maximum error $$\|f-\widetilde f\|_\infty=\max_{\mathbf x\in[0,1]^d}|f(\mathbf x)-\widetilde f(\mathbf x)|.$$

\subsection{Fast deep approximation of squaring and multiplication}\label{sec:squaring}

Our first key result shows that ReLU networks with unconstrained depth can very efficiently approximate the function $f(x)=x^2$ (more efficiently than any fixed-depth network, as we will see in Section \ref{sec:slow}). Our construction uses the ``sawtooth'' function that has previously appeared in the paper \cite{telgarsky2015representation}.   

\begin{prop}\label{th:x2}
The function $f(x)=x^2$ on the segment $[0,1]$ can be approximated with any error $\epsilon>0$  by a ReLU network having the depth and the number of weights and computation units $O(\ln (1/\epsilon))$.
\end{prop}
\begin{proof}

Consider the ``tooth'' (or ``mirror'') function $g:[0,1]\to [0,1],$
$$g(x)=\begin{cases}2x, & x<\frac{1}{2},\\ 2(1-x), & x\ge\frac{1}{2},\end{cases}$$
and the iterated  functions
$$g_s(x) = \underbrace{g\circ g\circ\cdots \circ g}_s(x).$$
Telgarsky has shown (see Lemma 2.4 in \cite{telgarsky2015representation}) that $g_s$ is a ``sawtooth'' function with $2^{s-1}$ uniformly distributed ``teeth'' (each application of $g$ doubles the number of teeth):
$$g_s(x)=
\begin{cases}2^s\big(x-\frac{2k}{2^s}\big), & x\in\big[\frac{2k}{2^s}, \frac{2k+1}{2^s}], k=0,1,\ldots, 2^{s-1}-1,\\ 
2^s\big(\frac{2k}{2^s}-x\big), & x\in\big[\frac{2k-1}{2^s}, \frac{2k}{2^s}], k=1,2,\ldots, 2^{s-1},\end{cases}$$
(see Fig. \ref{fig:prop1a}). Our key observation now is that the function $f(x)=x^2$ can be approximated by linear combinations of the functions $g_s$. Namely, let $f_m$ be the piece-wise linear interpolation of $f$ with $2^{m}+1$ uniformly distributed breakpoints $\frac{k}{2^m}, k=0, \ldots, 2^m$:
$$f_m\Big(\frac{k}{2^m}\Big) = \Big(\frac{k}{2^m}\Big)^2,\quad k=0,\ldots, 2^m$$
(see Fig. \ref{fig:prop1b}). The function $f_m$ approximates $f$ with the error $\epsilon_m =2^{-2m-2}$. Now note that refining the interpolation from $f_{m-1}$ to $f_m$ amounts to adjusting it by a function proportional to a sawtooth function: $$f_{m-1}(x)-f_{m}(x)=\frac{g_m(x)}{2^{2m}}.$$ Hence
$$f_m(x)=x-\sum_{s=1}^m \frac{g_s(x)}{2^{2s}}.$$
Since $g$ can be implemented by a finite ReLU network (as $g(x)=2\sigma(x)-4\sigma\big(x-\frac{1}{2}\big)+2\sigma(x-1)$) and since construction of $f_m$ only involves $O(m)$ linear operations and compositions of $g$, we can implement $f_m$ by a ReLU network having depth and the number of weights and computation units all being $O(m)$ (see Fig. \ref{fig:prop1c}). This implies the claim of the proposition.  

\begin{figure}
\begin{center}
    \begin{subfigure}[b]{0.3\textwidth}
        \includegraphics[width=\textwidth, trim={0mm 0mm 0mm 0mm},clip]{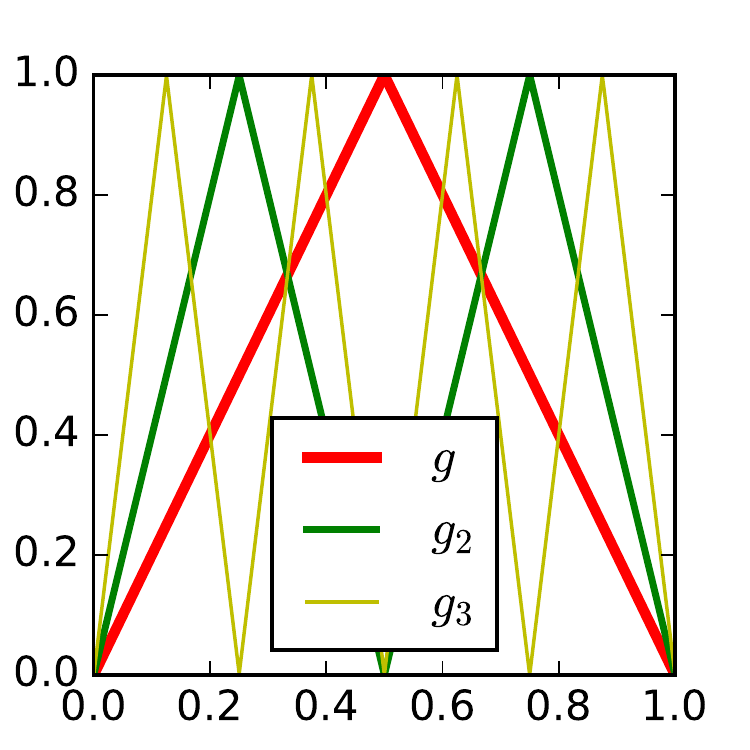}
        \caption{}
        \label{fig:prop1a}
    \end{subfigure}    
    \begin{subfigure}[b]{0.3\textwidth}
        \includegraphics[width=\textwidth, trim={0mm 0mm 0mm 0mm},clip]{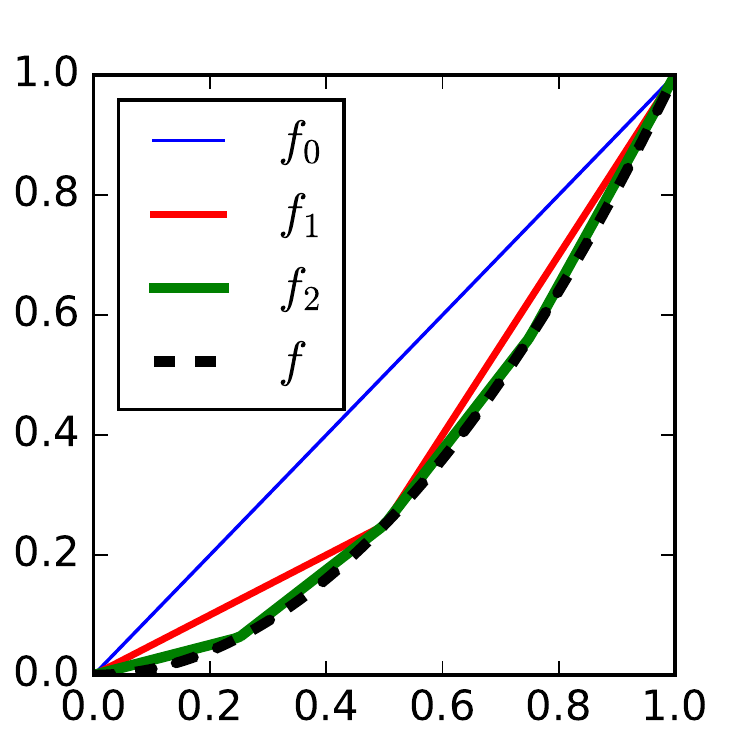}
        \caption{}
        \label{fig:prop1b}
    \end{subfigure}
    \begin{subfigure}[b]{0.3\textwidth}
        \includegraphics[width=\textwidth, trim={15mm 10mm 15mm 10mm},clip]{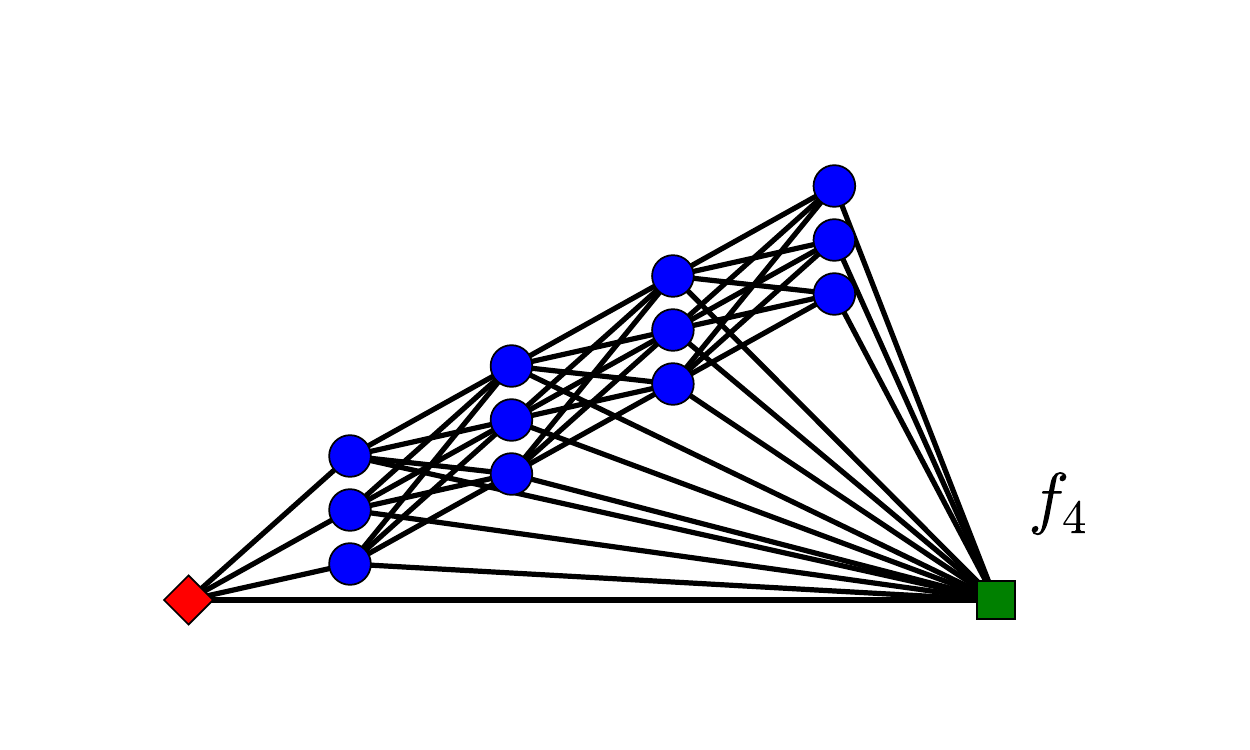}
        \caption{}
        \label{fig:prop1c}
    \end{subfigure}
\caption{Fast approximation of the function $f(x)=x^2$ from Proposition \ref{th:x2}: (a) the ``tooth'' function $g$ and the iterated ``sawtooth'' functions $g_2, g_3$; (b) the approximating functions $f_m$; (c) the network architecture for $f_4$.}
\label{fig:prop1}
\end{center}
\end{figure}

\end{proof}

Since  
\begin{equation}\label{fg}
xy=\frac{1}{2}((x+y)^2-x^2-y^2),
\end{equation} 
we can use Proposition \ref{th:x2} to efficiently implement accurate multiplication in a ReLU network. The implementation will depend on the required accuracy and the magnitude of the multiplied quantities.

\begin{prop}\label{th:fg}
Given $M>0$ and $\epsilon \in (0,1)$, there is a ReLU network $\eta$ with two input units that implements a function $\widetilde\times:\mathbb R^2\to\mathbb R$ so that
\begin{enumerate}
\item[a)] for any inputs $x,y$, if $|x|\le M$ and $|y|\le M,$ then $|\widetilde\times(x,y)-xy|\le \epsilon$;
\item[b)] if $x=0$ or $y=0$, then $\widetilde\times(x,y)=0$;
\item[c)] the depth and the number of weights and computation units in $\eta$ is not greater than $c_1\ln(1/\epsilon)+c_2$ with an absolute constant $c_1$ and a constant $c_2=c_2(M)$.
\end{enumerate}
\end{prop}
\begin{proof}
Let $\widetilde f_{\mathrm{sq}, \delta}$ be the approximate squaring function from Proposition \ref{th:x2} such that $\widetilde f_{\mathrm{sq}, \delta}(0)=0$ and $|\widetilde f_{\mathrm{sq}, \delta}(x)-x^2|<\delta$ for $x\in [0,1]$. Assume without loss of generality that $M\ge 1$ and set 
\begin{equation}\label{etafg}
\widetilde\times(x,y)=\frac{M^2}{8}\bigg(\widetilde f_{\mathrm{sq}, \delta}\Big(\frac{|x+y|}{2M}\Big)-\widetilde f_{\mathrm{sq}, \delta}\Big(\frac{|x|}{2M}\Big)-\widetilde f_{\mathrm{sq}, \delta}\Big(\frac{|y|}{2M}\Big)\bigg),
\end{equation}
where $\delta=\frac{8\epsilon}{3M^2}$. Then property b) is immediate and a) follows easily using expansion \eqref{fg}. To conclude c), observe that  computation \eqref{etafg} consists of three instances of $\widetilde f_{\mathrm{sq}, \delta}$  and finitely many linear and ReLU operations, so, using Proposition \ref{th:x2}, we can implement $\widetilde\times$ by a ReLU network such that its depth and the number of computation units and weights are  $O(\ln(1/\delta))$, i.e. are $O(\ln(1/\epsilon)+\ln M)$.
\end{proof}

\subsection{Fast deep approximation of general smooth functions}\label{sec:gensmooth}
In order to formulate our general result, Theorem \ref{th:gensmooth}, we consider the Sobolev spaces $\mathcal W^{n,\infty}([0,1]^d)$ with $n=1,2,\ldots$ Recall that $\mathcal W^{n,\infty}([0,1]^d)$ is defined as the space of functions on $[0,1]^d$ lying in $L^\infty$ along with their weak derivatives  up to order $n$. The norm in $\mathcal W^{n,\infty}([0,1]^d)$ can be defined by
$$\|f\|_{\mathcal W^{n,\infty}([0,1]^d)}= \max_{\mathbf{n}:|\mathbf{n}|\le n}\esssup_{\mathbf{x}\in [0,1]^d}|D^\mathbf{n}f(\mathbf{x})|,$$
where $\mathbf{n}=(n_1,\ldots,n_d)\in \{0,1,\ldots\}^d$, $|\mathbf{n}|=n_1+\ldots+n_d$, and $D^{\mathbf n}f$ is the respective weak derivative. Here and in the sequel we denote vectors by boldface characters. The space $\mathcal W^{n,\infty}([0,1]^d)$ can be equivalently described as consisting of the functions from $C^{n-1}([0,1]^d)$ such that all their derivatives of order $n-1$ are Lipschitz continuous.

Throughout the paper, we denote by $F_{n,d}$ the unit ball in $\mathcal W^{n,\infty}([0,1]^d)$:
$$F_{n,d}=\{f\in \mathcal W^{n,\infty}([0,1]^d): \|f\|_{\mathcal W^{n,\infty}([0,1]^d)}\le 1\}.$$

Also, it will now be convenient to make a distinction between \emph{networks} and \emph{network architectures}: we define the latter as the former with unspecified weights. We say that a network architecture \emph{is capable of expressing any function from $F_{d,n}$ with error $\epsilon$} meaning that this can be achieved by some weight assignment. 

\begin{theor}\label{th:gensmooth}
For any $d,n$ and $\epsilon\in (0,1)$, there is a ReLU network architecture that 
\begin{enumerate}
\item is capable of expressing any function from $F_{d,n}$ with error $\epsilon$;
\item has the depth at most $c(\ln (1/\epsilon)+1)$ and at most $c\epsilon^{-d/n}(\ln (1/\epsilon)+1)$ weights and computation units, with some constant $c=c(d,n)$.
\end{enumerate}
\end{theor}

\begin{proof} The proof will consist of two steps.
We start with approximating $f$ by a sum-product combination $f_1$ of local Taylor polynomials and one-dimensional piecewise-linear functions. After that, we will use results of the previous section to approximate $f_1$ by a neural network.  

Let $N$ be a positive integer. Consider a partition of unity  formed by a grid of $(N+1)^d$ functions $\phi_{\mathbf m}$ on the domain $[0,1]^d$:   
$$\sum_{\mathbf{m}} \phi_\mathbf{m}(\mathbf x) \equiv 1, \quad \mathbf x\in [0,1]^d.$$ 
Here $\mathbf m = (m_1,\ldots,m_d)\in \{0,1,\ldots,N\}^d,$ and the function $\phi_{\mathbf m}$ is defined as the product
\begin{equation}\label{eq:phi}
\phi_{\mathbf m}(\mathbf x)=\prod_{k=1}^d \psi\Big(3N\big(x_k-\frac{m_k}{N}\big)\Big),
\end{equation}
where
$$\psi(x)=\begin{cases}
1, & |x|< 1,\\
0, & 2 < |x|, \\
2-|x|, & 1\le |x|\le 2 
\end{cases}$$
(see Fig. \ref{fig:partition}). Note that 
\begin{equation}\label{eq:phinorm}
\|\psi\|_\infty=1\text{ and }\|\phi_{\mathbf m}\|_\infty=1\;\;\forall \mathbf m 
\end{equation} and
\begin{equation}\label{eq:phisupp}
\operatorname{supp}\phi_{\mathbf m}\subset \Big\{\mathbf x: \Big|x_k-\frac{m_k}{N}\Big|<\frac{1}{N}\;\; \forall k\Big\}. 
\end{equation}
\begin{figure}
\begin{center}
\includegraphics[width=0.3\textwidth, trim= 0mm 0mm 0mm 0mm,clip]{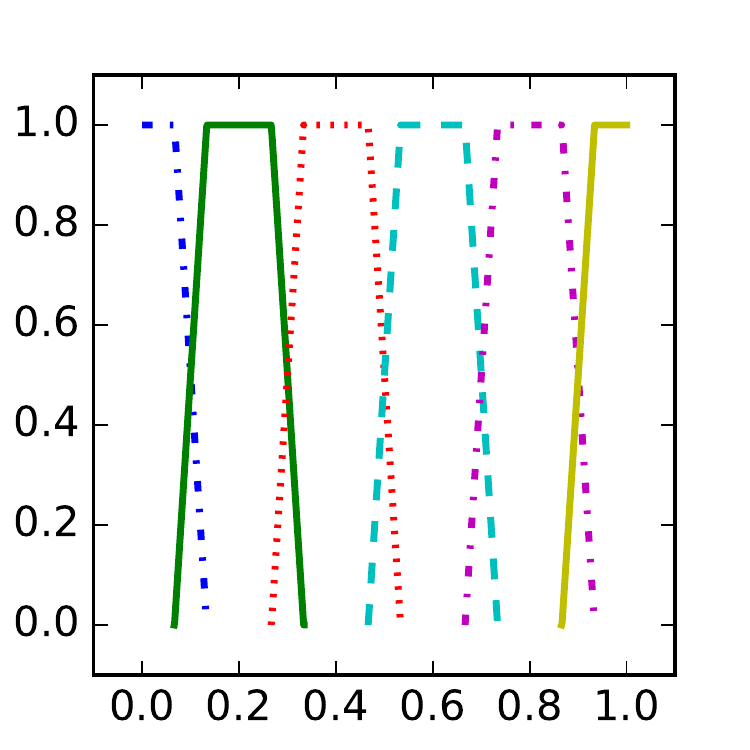}
\end{center}
\caption{Functions $(\phi_m)_{m=0}^5$ forming a partition of unity for $d=1, N=5$ in the proof of Theorem \ref{th:gensmooth}.}\label{fig:partition}
\end{figure}
For any $\mathbf m\in \{0,\ldots,N\}^d$, consider the degree-$(n-1)$ Taylor polynomial for the function $f$ at $\mathbf x=\frac{\mathbf m}{N}$:
\begin{equation}\label{p}
P_{\mathbf m}(\mathbf x)=\sum_{\mathbf n: |\mathbf n|<n}\frac{D^{\mathbf n}f}{{\mathbf n}!}\bigg|_{\mathbf x=\frac{\mathbf m}{N}} \Big(\mathbf x-\frac{\mathbf m}{N}\Big)^\mathbf{n},
\end{equation}
with the usual conventions $\mathbf n!=\prod_{k=1}^dn_k!$ and $(\mathbf x-\frac{\mathbf m}{N})^{\mathbf n}=\prod_{k=1}^d(x_k-\frac{m_k}{N})^{n_k}$.
Now define an approximation to $f$ by 
\begin{equation}\label{f1}
f_1=\sum_{\mathbf m\in\{0,\ldots,N\}^d}\phi_{\mathbf m}P_{\mathbf m}.
\end{equation}
We bound the approximation error using the Taylor expansion of $f$:
\begin{align*}
|f(\mathbf x)-f_1(\mathbf x)| &= \Big|\sum_{\mathbf m}\phi_{\mathbf m}(\mathbf x)(f(\mathbf x)-P_{\mathbf m}(\mathbf x))\Big| \\
&\le \sum_{\mathbf m: |x_k-\frac{m_k}{N}|<\frac{1}{N} \forall k}|f(\mathbf x)-P_{\mathbf m}(\mathbf x)| \\
&\le 2^d\max_{\mathbf m: |x_k-\frac{m_k}{N}|<\frac{1}{N} \forall k}|f(\mathbf x)-P_{\mathbf m}(\mathbf x)| \\
&\le \frac{2^d d^n}{n!}\Big(\frac{1}{N}\Big)^{n}\max_{\mathbf n: |\mathbf n|=n }\esssup_{\mathbf x\in[0,1]^d}|D^{\mathbf n} f(\mathbf x)| \\
&\le \frac{2^d d^n}{n!}\Big(\frac{1}{N}\Big)^{n}.
\end{align*}
Here in the second step we used the support property \eqref{eq:phisupp} and the bound \eqref{eq:phinorm}, in the third the observation that any $\mathbf x\in[0,1]^d$ belongs to the support of at most $2^d$ functions $\phi_\mathbf{m}$, in the fourth a standard bound for the Taylor remainder, and in the fifth the property $\|f\|_{\mathcal W^{n,\infty}([0,1]^d)}\le 1.$

It follows that if we choose 
\begin{equation}\label{N}
N=\Big\lceil\big(\frac{n!}{2^dd^n}\frac{\epsilon}{2}\big)^{-1/n}\Big\rceil\end{equation}
(where $\lceil\cdot\rceil$ is the ceiling function), then 
\begin{equation}\label{eq:ffqeps}
\|f-f_1\|_\infty \le \frac{\epsilon}{2}.\end{equation} 
Note that, by \eqref{p} the coefficients of the polynomials $P_\mathbf{m}$ are uniformly bounded for all $f\in F_{d,n}$: 
\begin{equation}\label{p2}
P_{\mathbf m}(\mathbf x)=\sum_{\mathbf n: |\mathbf n|<n}a_{\mathbf{m},\mathbf{n}} \Big(\mathbf x-\frac{\mathbf m}{N}\Big)^\mathbf{n}, \quad |a_{\mathbf{m},\mathbf{n}}|\le 1.
\end{equation}

We have therefore reduced our task to the following: construct a network architecture capable of approximating with uniform error $\frac{\epsilon}{2}$ any function of the form \eqref{f1}, assuming that $N$ is given by \eqref{N} and the polynomials $P_{\mathbf m}$ are of the form \eqref{p2}.

Expand $f_1$ as 
\begin{equation}\label{eq:f1def}f_1(\mathbf x)=\sum_{\mathbf m\in\{0,\ldots,N\}^d}\sum_{\mathbf n:|\mathbf n|<n}a_{\mathbf m, \mathbf n}\phi_{\mathbf m}(\mathbf x)\big(\mathbf x-\tfrac{\mathbf m}{N}\big)^\mathbf{n}.\end{equation}
The expansion is a linear combination of not more than $d^n(N+1)^d$ terms $\mathbf \phi_{\mathbf m}(\mathbf x)(\mathbf x-\frac{\mathbf m}{N})^\mathbf{n}$. Each of these terms is a product of at most $d+n-1$ piece-wise linear univariate factors: $d$ functions $\psi(3Nx_k-3m_k)$ (see \eqref{eq:phi}) and at most $n-1$ linear expressions $x_k-\frac{m_k}{N}$. We can implement an approximation of this product by a neural network with the help of Proposition \ref{th:fg}. Specifically, let $\widetilde\times$ be the approximate multiplication from Proposition \ref{th:fg} for $M=d+n$ and some accuracy $\delta$ to be chosen later, and consider the approximation of the product $\mathbf \phi_{\mathbf m}(\mathbf x)(\mathbf x-\frac{\mathbf m}{N})^\mathbf{n}$ obtained by the chained application of $\widetilde\times$:
\begin{equation}\label{eq:chain}
\widetilde f_{\mathbf m, \mathbf n}(\mathbf x)=
\widetilde\times\big(\psi(3Nx_1-3m_1), \widetilde\times\big(\psi(3Nx_2-3m_2),\ldots,\widetilde\times\big(x_k-\tfrac{m_k}{N},\ldots\big)\ldots\big)\big).
\end{equation}
that Using statement c) of Proposition \ref{th:fg}, we see $\widetilde f_{\mathbf m, \mathbf n}$ can be implemented by a ReLU network with the depth and the number of weights and computation units not larger than $(d+n)c_1 \ln(1/\delta),$ for some constant $c_1=c_1(d,n)$. 

Now we estimate the error of this approximation. Note that we have $|\psi(3Nx_k-3m_k)|\le 1$ and $|x_k-\frac{m_k}{N}|\le 1$ for all $k$ and all $\mathbf x\in[0,1]^d$. By statement a) of Proposition \ref{th:fg}, if $|a|\le 1$ and $|b|\le M$, then $|\widetilde\times(a,b)|\le |b|+\delta$. Repeatedly applying this observation to all approximate multiplications in \eqref{eq:chain} while assuming $\delta<1$, we see that the arguments of all these multiplications are bounded by our $M$ (equal to $d+n$) and the statement a) of Proposition \ref{th:fg} holds for each of them. We then have
\begin{equation}
\begin{aligned}\label{eq:fmn}
\big|\widetilde f_{\mathbf m, \mathbf n}(\mathbf x)-&\phi_{\mathbf m}(\mathbf x)\big(\mathbf x-\tfrac{\mathbf m}{N}\big)^\mathbf{n}\big| \\
=&\big|\widetilde\times\big(\psi(3Nx_1-3m_1), \widetilde\times\big(\psi(3Nx_2-3m_2),\widetilde\times\big(\psi(3Nx_3-3m_3),\ldots\big)\big)\big)\\
&-\psi(3Nx_1-3m_1)\psi(3Nx_2-3m_2)\psi(3Nx_3-3m_3)\ldots\big|\\
\le&\big|\widetilde\times\big(\psi(3Nx_1-3m_1), 
\widetilde\times\big(\psi(3Nx_2-3m_2),
\widetilde\times\big(\psi(3Nx_3-3m_3),\ldots\big)\big)\big)\\
&-\psi(3Nx_1-3m_1)\cdot\widetilde\times\big(\psi(3Nx_2-3m_2),\widetilde\times\big(\psi(3Nx_3-3m_3),\ldots\big)\big)\big|\\
&+|\psi(3Nx_1-3m_1)|\cdot\big|\widetilde\times\big(\psi(3Nx_2-3m_2),
\widetilde\times\big(\psi(3Nx_3-3m_3),\ldots\big)\big)\\
&-\psi(3Nx_2-3m_2)\cdot\widetilde\times\big(\psi(3Nx_3-3m_3),\ldots\big)\big|\\
&+\ldots\\
\le&(d+n)\delta.
\end{aligned}
\end{equation}
Moreover, by statement b) of Proposition \ref{th:fg},
\begin{equation}\label{eq:xi0}
\widetilde f_{\mathbf m, \mathbf n}(\mathbf x)=\phi_{\mathbf m}(\mathbf x)\big(\mathbf x-\tfrac{\mathbf m}{N}\big)^\mathbf{n}, \quad \mathbf x\notin \operatorname{supp}\phi_{\mathbf m}.
\end{equation}
Now we define the full approximation by
\begin{equation}\label{eq:wffmn}\widetilde f=\sum_{\mathbf m\in\{0,\ldots,N\}^d}\sum_{\mathbf n:|\mathbf n|<n}a_{\mathbf m, \mathbf n}\widetilde f_{\mathbf m, \mathbf n}.\end{equation}
We estimate the approximation error of $\widetilde f$:
\begin{align*}|\widetilde f(\mathbf x)-f_1(\mathbf x)| &= \bigg|\sum_{\mathbf m\in\{0,\ldots,N\}^d}\sum_{\mathbf n:|\mathbf n|<n}a_{\mathbf m, \mathbf n}\Big(\widetilde f_{\mathbf m, \mathbf n}(\mathbf x)-\phi_{\mathbf m}(\mathbf x)\big(\mathbf x-\tfrac{\mathbf m}{N}\big)^\mathbf{n}\Big)\bigg|\\
&= \bigg|\sum_{\mathbf m: \mathbf x\in \operatorname{supp}\phi_{\mathbf m}}\sum_{\mathbf n:|\mathbf n|<n}a_{\mathbf m, \mathbf n}\Big(\widetilde f_{\mathbf m, \mathbf n}(\mathbf x)-\phi_{\mathbf m}(\mathbf x)\big(\mathbf x-\tfrac{\mathbf m}{N}\big)^\mathbf{n}\Big)\bigg|\\
&\le 2^d \max_{\mathbf m: x\in \operatorname{supp}\phi_{\mathbf m}}\sum_{\mathbf n:|\mathbf n|<n}\Big|\widetilde f_{\mathbf m, \mathbf n}(\mathbf x)-\phi_{\mathbf m}(\mathbf x)\big(\mathbf x-\tfrac{\mathbf m}{N}\big)^\mathbf{n}\Big|\\
&\le 2^dd^n(d+n)\delta,
\end{align*}
where in the first step we use expansion \eqref{eq:f1def}, in the second the identity \eqref{eq:xi0}, in the third the bound $|a_{\mathbf m, \mathbf n}|\le 1$ and the fact that $\mathbf x\in\operatorname{supp}\phi_{\mathbf m}$ for at most $2^d$ functions $\phi_{\mathbf m},$ and in the fourth the bound \eqref{eq:fmn}. It follows that if we choose \begin{equation}\label{eq:deltaeps}
\delta=\frac{\epsilon}{2^{d+1}d^n(d+n)},\end{equation} then $\|\widetilde f-f_1\|_\infty\le\frac{\epsilon}{2}$ and hence, by \eqref{eq:ffqeps}, $$\|\widetilde f-f\|_\infty\le \|\widetilde f-f_1\|_\infty+\|f_1-f\|_\infty\le \frac{\epsilon}{2}+\frac{\epsilon}{2}\le\epsilon.$$

On the other hand, note that by \eqref{eq:wffmn}, $\widetilde f$ can be implemented by a network consisting of parallel subnetworks that compute each of $\widetilde f_{\mathbf m, \mathbf n}$; the final output is obtained by weighting the outputs of the subnetworks with the weights $a_{\mathbf m, \mathbf n}$. The architecture of the full network does not depend on $f$; only the weights $a_{\mathbf m, \mathbf n}$ do. As already shown, each of these subnetworks has not more than $c_1\ln(1/\delta)$ layers, weights and computation units,  with some constant $c_1=c_1(d,n)$. There are not more than $d^n(N+1)^d$ such subnetworks. Therefore, the full network for $\widetilde f$ has not more than $c_1\ln(1/\delta)+1$ layers and $d^n(N+1)^d(c_1\ln(1/\delta)+1)$ weights and computation units. With $\delta$ given by \eqref{eq:deltaeps} and $N$ given by \eqref{N}, we obtain the claimed complexity bounds.
\end{proof}

\subsection{Faster approximations using adaptive network architectures}
\label{sec:faster}
Theorem \ref{th:gensmooth} provides an upper bound for the approximation complexity in the case when the same network architecture is used to approximate all functions in $F_{d,n}$. We can consider an alternative, ``adaptive architecture'' scenario where not only the weights, but also the architecture is adjusted to the approximated function. We expect, of course, that this would decrease the complexity of the resulting architectures, in general (at the price of needing to find the appropriate architecture). In this section we show that we can indeed obtain better upper bounds in this scenario.

For simplicity, we will only consider the case $d=n=1$. Then, $\mathcal W^{n,\infty}([0,1]^d)$ is the space of Lipschitz functions on the segment $[0,1]$. The set $F_{1,1}$ consists of functions $f$ having both $\|f\|_\infty$ and the Lipschitz constant bounded by 1. Theorem \ref{th:gensmooth} provides  an upper bound $O(\frac{\ln(1/\epsilon)}{\epsilon})$ for the number of weights and computation units, but in this special case there is in fact a better bound $O(\frac{1}{\epsilon})$ obtained simply by piece-wise interpolation. 

Namely, given $f\in F_{1,1}$ and $\epsilon>0$, set $T=\lceil\frac{1}{\epsilon}\rceil$ and let $\widetilde f$ be the piece-wise interpolation of $f$ with $T+1$ uniformly spaced breakpoints $(\frac{t}{T})_{t=0}^T$ (i.e., $\widetilde f(\frac{t}{T})=f(\frac{t}{T}), t=0,\ldots,T$).  The function $\widetilde f$ is also Lipschitz with constant 1 and hence $\|f-\widetilde f\|_\infty\le\frac{1}{T}\le\epsilon$ (since for any $x\in [0,1]$ we can find $t$ such that $|x-\frac{t}{T}|\le\frac{1}{2T}$ and then $|f(x)-\widetilde f(x)|\le |f(x)-f(\frac{t}{T})|+|\widetilde f(\frac{t}{T})-\widetilde f(x)|\le 2\cdot\frac{1}{2T}=\frac{1}{T}$). At the same time, the function $\widetilde f$ can be expressed in terms of the ReLU function $\sigma$ by $$\widetilde f(x)=b+\sum_{t=0}^{T-1}w_t \sigma\Big(x-\frac{t}{T}\Big)$$ with some coefficients $b$ and $(w_t)_{t=0}^{T-1}$. This expression can be viewed as a special case of the depth-3 ReLU network with $O(\frac{1}{\epsilon})$ weights and computation units. 

We show now how the bound $O(\frac{1}{\epsilon})$ can be improved by using adaptive architectures. 

\begin{theor}\label{th:faster}
For any $f\in F_{1,1}$ and $\epsilon\in (0,\frac{1}{2})$, there exists a depth-6 ReLU network $\eta$ (with architecture depending on $f$) that provides an $\epsilon$-approximation of $f$ while having not more than $\frac{c}{\epsilon\ln(1/\epsilon)}$ weights, connections and computation units. Here $c$ is an absolute constant.  
\end{theor}
\begin{proof}
We first explain the idea of the proof. We start with interpolating $f$ by a piece-wise linear function, but not on the length scale $\epsilon$ -- instead, we do it on a coarser length scale $m\epsilon$, with some $m=m(\epsilon)>1$. We then create a ``cache'' of auxiliary subnetworks that we use to fill in the details and go down to the scale $\epsilon$, in each of the $m\epsilon$-subintervals. This allows us to reduce the amount of computations for small $\epsilon$ because the complexity of the cache only depends on $m$. 
The assignment of cached subnetworks to the subintervals is encoded in the network architecture and depends on the function $f$. We optimize $m$ by balancing the complexity of the cache  with that of the initial coarse approximation. This leads to $m\sim \ln(1/\epsilon)$ and hence to the reduction of the total complexity of the network by a factor  $\sim\ln(1/\epsilon)$ compared to the simple piece-wise linear approximation on the scale $\epsilon$. This construction is inspired by a similar argument used to prove the $O(2^n/n)$ upper bound for the complexity of Boolean circuits implementing $n$-ary functions \cite{shannon1949synthesis}.

The proof becomes simpler if, in addition to the ReLU function $\sigma$, we are allowed to use the activation function
\begin{equation}\label{eq:rho}
\rho(x)=\begin{cases}x, & x\in[0,1),\\ 0, & x\notin[0,1)\end{cases}
\end{equation}
in our neural network. Since $\rho$ is discontinuous, we cannot just use Proposition \ref{th:piecewise} to replace $\rho$-units by $\sigma$-units.
We will first prove the analog of the claimed result for the model including $\rho$-units, and then we will show how to construct a purely ReLU nework.
\begin{lemma}\label{th:rho}
For any $f\in F_{1,1}$ and $\epsilon\in (0,\frac{1}{2})$, there exists a depth-5 network including $\sigma$-units and $\rho$-units, that provides an $\epsilon$-approximation of $f$ while having not more than $\frac{c}{\epsilon\ln(1/\epsilon)}$ weights, where $c$ is an absolute constant.  
\end{lemma}
\begin{proof}
Given $f\in F_{1,1}$, we will construct an approximation $\widetilde f$ to $f$ in the form
$$\widetilde f= \widetilde f_1+\widetilde f_2.$$
Here, $\widetilde f_1$ is the piece-wise linear interpolation of $f$ with the breakpoints $\{\frac{t}{T}\}_{t=0}^T$, for some positive integer $T$ to be chosen later. Since $f$ is Lipschitz with constant 1, $\widetilde f_1$ is also Lipschitz with constant 1. We will denote by $I_t$ the intervals between the breakpoints:
$$I_t=\Big[\frac{t}{T}, \frac{t+1}{T}\Big),\quad t=0,\ldots,T-1.$$
We will now construct $\widetilde f_2$ as an approximation to the difference \begin{equation}\label{eq:f2}f_2=f-\widetilde f_1.\end{equation} 
Note that $f_2$ vanishes at the endpoints of the intervals $I_t$:
\begin{equation}\label{eq:f2pr1} f_2\Big(\frac{t}{T}\Big)=0,\;\; t=0,\ldots,T,
\end{equation}
and $f_2$ is Lipschitz with constant 2: 
\begin{equation}\label{eq:f2pr2} |f_2(x_1)-f_2(x_2)|\le 2|x_1-x_2|,
\end{equation}
since $f$ and $\widetilde f_1$ are Lipschitz with constant 1.

To define $\widetilde f_2$, we first construct a set $\Gamma$ of cached functions. Let $m$ be a positive integer to be chosen later. Let $\Gamma$ be the set of piecewise linear functions $\gamma: [0,1]\to\mathbb R$ with the breakpoints $\{\frac{r}{m}\}_{r=0}^m$ and the properties 
$$\gamma(0)=\gamma(1)=0$$
and
$$\gamma\Big(\frac{r}{m}\Big)-\gamma\Big(\frac{r-1}{m}\Big)\in \Big\{-\frac{2}{m}, 0, \frac{2}{m}\Big\},\quad r =1, \ldots, m.$$
Note that the size $|\Gamma|$ of $\Gamma$ is not larger than $3^m$.

If  $g:[0,1]\to \mathbb R$ is any Lipschitz function with constant 2 and $g(0)=g(1)=0$, then $g$ can be approximated by some $\gamma\in \Gamma$ with error not larger than $\frac{2}{m}$: namely, take $\gamma(\frac{r}{m})=\frac{2}{m}\lfloor g(\frac{r}{m})/\frac{2}{m}\rfloor$.

Moreover, if $f_2$ is defined by \eqref{eq:f2}, then, using \eqref{eq:f2pr1}, \eqref{eq:f2pr2}, on each interval $I_t$ the function $f_2$ can be approximated with error not larger than $\frac{2}{Tm}$ by a properly rescaled function $\gamma\in\Gamma$. Namely, for each $t=0,\ldots,T-1$ we can define the function $g$ by $g(y)=Tf_2(\frac{t+y}{T})$. Then it is Lipschitz with constant 2 and $g(0)=g(1)=0$, so we can find $\gamma_t\in\Gamma$ such that
\begin{equation*}
\sup_{y\in [0,1)}\Big|Tf_2\Big(\frac{t+y}{T}\Big)-\gamma_t(y)\Big| \le \frac{2}{m}.
\end{equation*}
This can be equivalently written as 
\begin{equation*}
\sup_{x\in I_t}\Big|f_2(x)-\frac{1}{T}\gamma_t(Tx-t)\Big| \le \frac{2}{Tm}.
\end{equation*}
Note that the obtained assignment $t\mapsto \gamma_t$ is not injective, in general ($T$ will be much larger than $|\Gamma|$).

We can then define $\widetilde f_2$ on the whole $[0,1)$ by
\begin{equation}\label{eq:f2i}
\widetilde f_2(x)=\frac{1}{T}\gamma_t(Tx-t),\quad x\in I_t, \quad t=0,\ldots,T-1.
\end{equation}
This $\widetilde f_2$ approximates $f_2$ with error $\frac{2}{Tm}$ on $[0,1)$: 
\begin{equation}\label{eq:f22tm}
\sup_{x\in[0,1)}|f_2(x)-\widetilde f_2(x)|\le \frac{2}{Tm},
\end{equation}
and hence, by \eqref{eq:f2}, for the full approximation $\widetilde f=\widetilde f_1+\widetilde f_2$ we will also have
\begin{equation}\label{eq:f2tm}
\sup_{x\in[0,1)}|f(x)-\widetilde f(x)|\le \frac{2}{Tm}.
\end{equation}
Note that the approximation $\widetilde f_2$ has properties analogous to \eqref{eq:f2pr1}, \eqref{eq:f2pr2}: 
\begin{equation}\label{eq:wf2pr1} \widetilde f_2\Big(\frac{t}{T}\Big)=0,\quad t=0,\ldots,T,
\end{equation}
\begin{equation}\label{eq:wf2pr2} |\widetilde f_2(x_1)-\widetilde f_2(x_2)|\le 2|x_1-x_2|,
\end{equation}
in particular, $\widetilde f_2$ is continuous on $[0,1)$.

We will now rewrite $\widetilde f_2$ in a different form interpretable as a computation by a neural network. Specifically, using our additional activation function $\rho$ given by \eqref{eq:rho}, we can express $\widetilde f_2$  as
\begin{equation}\label{eq:wf2alt0}\widetilde f_2(x)
=\frac{1}{T}\sum_{\gamma\in \Gamma}\gamma\Big(\sum_{t: \gamma_t=\gamma} \rho(Tx-t)\Big).
\end{equation}
Indeed, given $x\in[0,1)$, observe that all the terms in the inner sum vanish except for the one corresponding to the $t$ determined by the condition $x\in I_t$. For this particular $t$ we have $\rho(Tx-t)=Tx-t$. Since $\gamma(0)=0$, we conclude that \eqref{eq:wf2alt0} agrees with \eqref{eq:f2i}.

Let us also expand $\gamma\in\Gamma$ over the basis of shifted ReLU functions:
$$\gamma(x) = \sum_{r=0}^{m-1}c_{\gamma,r}\sigma\Big(x-\frac{r}{m}\Big),\quad x\in[0,1].$$
Substituting this expansion in \eqref{eq:wf2alt0}, we finally obtain
\begin{equation}\label{eq:wf2alt}\widetilde f_2(x)
=\frac{1}{T}\sum_{\gamma\in \Gamma}\sum_{r=0}^{m-1}c_{\gamma,r}\sigma\Big(\sum_{t: \gamma_t=\gamma} \rho(Tx-t)-\frac{r}{m}\Big).
\end{equation} 

Now consider the implementation of $\widetilde f$ by a neural nework. The term $\widetilde f_1$ can clearly be implemented by a depth-3 ReLU network using $O(T)$ connections and computation units. The term $\widetilde f_2$ can be implemented by a depth-5 network with $\rho$- and $\sigma$-units as follows (we denote a computation unit by $Q$ with a superscript indexing the layer and a subscript indexing the unit within the layer).
\begin{enumerate}
\item The first layer contains the single input unit $Q^{(1)}$.
\item The second layer contains $T$ units $(Q^{(2)}_t)_{t=1}^{T}$ 
computing 
$Q^{(2)}_t=\rho(TQ^{(1)}-t).$
\item The third layer contains $|\Gamma|$ units $(Q^{(3)}_\gamma)_{\gamma\in\Gamma}$ computing $Q^{(3)}_\gamma=\sigma(\sum_{t: \gamma_t=\gamma} Q^{(2)}_t)$.  This is equivalent to $Q^{(3)}_\gamma=\sum_{t: \gamma_t=\gamma} Q^{(2)}_t$, because $Q^{(2)}_t\ge 0$. 
\item The fourth layer contains $m|\Gamma|$ units $(Q^{(4)}_{\gamma,r})_{\stackrel{\gamma\in\Gamma}{r=0,\ldots,m-1}}$ computing $Q^{(4)}_{\gamma,r}=\sigma(Q^{(3)}_{\gamma}-\frac{r}{m}).$
\item The final layer consists of a single output unit $Q^{(5)}=\sum_{\gamma\in\Gamma}\sum_{r=0}^{m-1}\frac{c_{\gamma,r}}{T}Q^{(4)}_{\gamma,r}.$
\end{enumerate}
Examining this network, we see that the total number of connections and units in it is $O(T+m|\Gamma|)$ and hence is $O(T+m3^m)$. This also holds for the full network implementing $\widetilde f=\widetilde f_1+\widetilde f_2$, since the term $\widetilde f_1$ requires even fewer layers, connections and units. The output units of the subnetworks for $\widetilde f_1$ and $\widetilde f_2$ can be merged into the output unit for $\widetilde f_1+\widetilde f_2$, so the depth of the full network is the maximum of the depths of the networks implementing $\widetilde f_1$ and $\widetilde f_2$, i.e., is 5 (see Fig. \ref{fig:cache}).

\begin{figure}
\begin{center}
\includegraphics[width=0.75\textwidth, trim={28mm 17mm 49mm 20mm},clip]{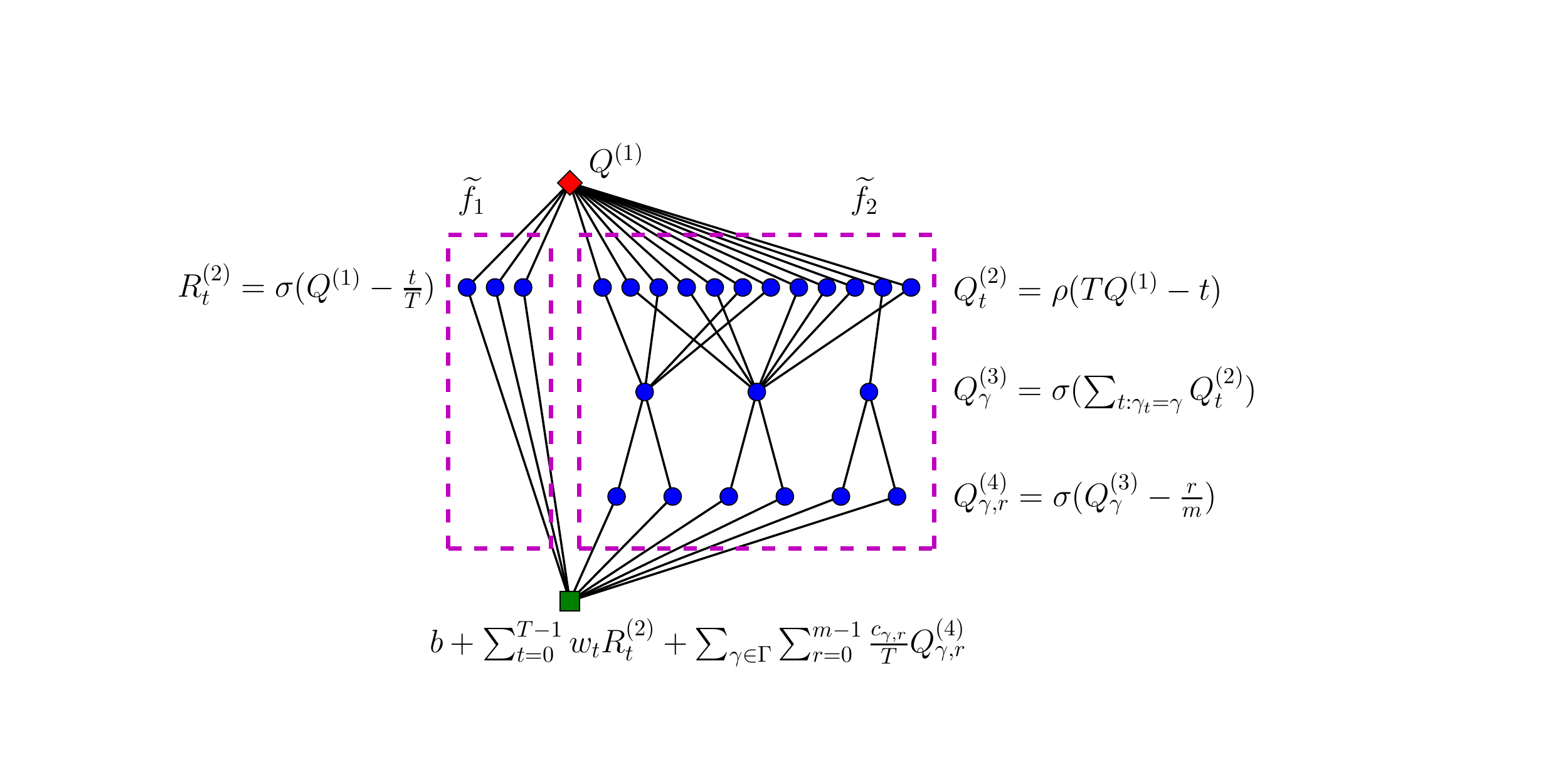}
\caption{Architecture of the network implementing the function $\widetilde f=\widetilde f_1+\widetilde f_2$ from Lemma~\ref{th:rho}. 
}
\label{fig:cache}
\end{center}
\end{figure}

Now, given $\epsilon\in(0,\frac{1}{2})$, take $m=\lceil\frac{1}{2}\log_3(1/\epsilon)\rceil$ and $T=\lceil\frac{2}{m\epsilon}\rceil.$ Then, by \eqref{eq:f2tm}, the approximation error $\max_{x\in[0,1]}|f(x)-\widetilde f(x)|\le\frac{2}{Tm}\le \epsilon$, while $T+m3^m =O(\frac{1}{\epsilon\ln(1/\epsilon)})$, which implies the claimed complexity bound.
\end{proof}

We show now how to modify the constructed network so as to remove $\rho$-units. We only need to modify the $\widetilde f_2$ part of the network. We will show that for any $\delta>0$ we can replace $\widetilde f_2$ with a function $\widetilde{f}_{3,\delta}$ (defined below) that 
\begin{itemize}
\item[a)] obeys the following analog of approximation bound \eqref{eq:f22tm}:
\begin{equation}\label{eq:f2f3}\sup_{x\in[0,1]}|f_2(x)-\widetilde{f}_{3,\delta}(x)|\le \frac{8\delta}{T}+\frac{2}{Tm},
\end{equation}
\item[b)] and is implementable by a depth-6 ReLU network having complexity $c(T+m3^m)$ with an absolute constant $c$ independent of $\delta$.
\end{itemize}
Since $\delta$ can be taken arbitrarily small, the Theorem then follows by arguing as in Lemma \ref{th:rho}, only with $\widetilde f_2$ replaced by $\widetilde{f}_{3,\delta}$.

As a first step, we approximate $\rho$ by a continuous piece-wise linear function $\rho_\delta$, with a small $\delta>0$:
\begin{equation*}
\rho(y)=\begin{cases}
y, & y\in[0,1-\delta),\\ 
\frac{1-\delta}{\delta}(1-y), & y\in[1-\delta,1),\\ 
0, & y\notin[0,1).\end{cases}
\end{equation*}
Let $\widetilde f_{2,\delta}$ be defined as $\widetilde f_2$ in \eqref{eq:wf2alt}, but with $\rho$ replaced by $\rho_\delta$:
\begin{equation*}
\widetilde f_{2,\delta}(x)
=\frac{1}{T}\sum_{\gamma\in \Gamma}\sum_{r=0}^{m-1}c_{\gamma,r}\sigma\Big(\sum_{t: \gamma_t=\gamma} \rho_\delta(Tx-t)-\frac{r}{m}\Big).
\end{equation*}
Since $\rho_\delta$ is a continuous piece-wise linear function with three breakpoints, we can express it via the ReLU function, and hence implement $\widetilde f_{2,\delta}$ by a purely ReLU network, as in Proposition \ref{th:piecewise}, and the complexity of the implementation does not depend on $\delta$. However, replacing $\rho$ with $\rho_\delta$ affects the function $\widetilde f_2$ on the intervals $(\frac{t-\delta}{T},\frac{t}{T}], t=1,\ldots,T$, introducing there a large error (of magnitude $O(\frac{1}{T})$). But recall that both $f_2$ and $\widetilde f_2$ vanish at the points $\frac{t}{T},t=0,\ldots,T,$ by \eqref{eq:f2pr1}, \eqref{eq:wf2pr1}. We can then largely remove this newly introduced error by simply suppressing $\widetilde f_{2,\delta}$ near the points  $\frac{t}{T}$.  

Precisely, consider the continuous piece-wise linear function
$$\phi_\delta(y)=
\begin{cases}
0, & y\notin[0,1-\delta),\\
\frac{y}{\delta}, & y\in[0, \delta),\\
1, & y\in[\delta,1-2\delta),\\
\frac{1-\delta-y}{\delta}, &y\in[1-2\delta,1-\delta)
\end{cases}$$
and the full comb-like filtering function
$$\Phi_\delta(x)=\sum_{t=0}^{T-1}\phi_\delta(Tx-t).$$
Note that $\Phi_\delta$ is continuous piece-wise linear with $4T$ breakpoints, and $0\le \Phi_\delta(x)\le 1$. 
We then define  our final modification of $\widetilde f_2$ as 
\begin{equation}\label{eq:f3delta}
\widetilde{f}_{3,\delta}(x)=\sigma\Big(\widetilde{f}_{2,\delta}(x)+2\Phi_\delta(x)-1\Big)-\sigma\Big(2\Phi_\delta(x)-1\Big).\end{equation}
\begin{lemma} The function $\widetilde{f}_{3,\delta}$ obeys the bound \eqref{eq:f2f3}.
\end{lemma} 
\begin{proof} Given $x\in[0,1)$, let $t\in\{0,\ldots,T-1\}$ and $y\in [0,1)$ be determined from the representation $x=\frac{t+y}{T}$   (i.e., $y$ is the relative position of $x$ in the respective interval $I_t$). Consider several possibilities for $y$: 
\begin{enumerate}
\item $y\in[1-\delta,1]$. In this case $\Phi_\delta(x)=0$. Note that \begin{equation}\label{eq:f2dl1}\sup_{x\in[0,1]}|\widetilde f_{2,\delta}(x)|\le 1,\end{equation} because, by construction, $\sup_{x\in[0,1]}|\widetilde f_{2,\delta}(x)|\le \sup_{x\in[0,1]}|\widetilde f_{2}(x)|$, and $\sup_{x\in[0,1]}|\widetilde f_{2}(x)|\le 1$ by \eqref{eq:wf2pr1}, \eqref{eq:wf2pr2}. It follows that both terms in  \eqref{eq:f3delta} vanish, i.e., $\widetilde{f}_{3,\delta}(x)=0$. But, since $f_2$ is Lipschitz with constant 2 by \eqref{eq:f2pr2} and $f_2(\frac{t+1}{T})=0$, we have $|f_2(x)|\le |f_2(x)-f_2(\frac{t+1}{T})|\le \frac{2|y-1|}{T}\le \frac{2\delta}{T}$. This implies $|f_2(x)-\widetilde{f}_{3,\delta}(x)|\le \frac{2\delta}{T}$.
\item $y\in[\delta,1-2\delta]$. In this case $\Phi_\delta(x)=1$ and $\widetilde{f}_{2,\delta}(x)=\widetilde{f}_{2}(x)$. Using \eqref{eq:f2dl1}, we find that $\widetilde{f}_{3,\delta}(x)=\widetilde{f}_{2,\delta}(x)=\widetilde{f}_{2}(x)$. It follows that $|f_2(x)-\widetilde{f}_{3,\delta}(x)|=|f_2(x)-\widetilde{f}_{2}(x)|\le\frac{2}{Tm}$.
\item $y\in[0,\delta]\cup[1-2\delta,1-\delta]$. In this case $\widetilde{f}_{2,\delta}(x)=\widetilde{f}_{2}(x)$. Since $\sigma$ is Lipschitz with constant 1, $|\widetilde{f}_{3,\delta}(x)|\le |\widetilde{f}_{2,\delta}(x)|=|\widetilde{f}_{2}(x)|$. Both $f_2$ and $\widetilde{f}_{2}$ are Lipschitz with constant 2 (by \eqref{eq:f2pr2},  \eqref{eq:wf2pr2}) and vanish at $\frac{t}{T}$ and $\frac{t+1}{T}$ (by \eqref{eq:f2pr1}, \eqref{eq:wf2pr1}). It follows that $$|f_2(x)-\widetilde{f}_{3,\delta}(x)|\le |f_2(x)|+|\widetilde{f}_{2}(x)|\le 2\begin{cases}2|x-\frac{t}{T}|, & y\in[0,\delta]\\ 2|x-\frac{t+1}{T}|, & y\in[1-2\delta,1-\delta]\end{cases}\le \frac{8\delta}{T}.$$  
\end{enumerate}
\end{proof}
It remains to verify the complexity property b) of the function  $\widetilde f_{3,\delta}$.
As already mentioned, $\widetilde f_{2,\delta}$ can be implemented by a depth-5 purely ReLU network with not more than $c(T+m3^m)$ weights, connections and computation units, where $c$ is an absolute constant independent of $\delta$. 
The function $\Phi_\delta$ can be implemented  by a shallow, depth-3 network with $O(T)$ units and connection. Then, computation of $\widetilde f_{3,\delta}$ can be implemented by a network including two subnetworks for computing $\widetilde f_{2,\delta}$ and $\Psi_\delta$, and an additional layer containing  two $\sigma$-units as written in \eqref{eq:f3delta}. We thus obtain 6 layers in the resulting full network and, choosing $T$ and $m$ in the same way as in Lemma \ref{th:rho}, obtain the bound $\frac{c}{\epsilon\ln(1/\epsilon)}$ for the number of its connections, weights, and computation units.  
\end{proof}

\section{Lower bounds}\label{sec:futility}

\subsection{Continuous nonlinear widths}\label{sec:dwidth}
The method of continuous nonlinear widths (\cite{devore1989optimal}) is a very general approach to the analysis of parameterized nonlinear approximations, based on the assumption of continuous selection of their parameters. We are interested in the following lower bound for the complexity of approximations in $\mathcal W^{n,\infty}([0,1]^d).$ 
\begin{theor}[\cite{devore1989optimal}, Theorem 4.2]\label{th:dwidth} Fix $d,n$. Let $W$ be a positive integer and $\eta:\mathbb R^W\to C([0,1]^d)$ be any mapping between the space $\mathbb R^W$ and the space $C([0,1]^d)$. Suppose that there is a continuous map  $\mathcal M:F_{d,n}\to \mathbb R^W$ such that $\|f-\eta(\mathcal M(f))\|_{\infty}\le \epsilon$ for all $f\in F_{d,n}.$ Then $W\ge c\epsilon^{-d/n},$ with some constant $c$ depending only on $n$.
\end{theor}
We apply this theorem by taking $\eta$ to be some ReLU network architecture, and $\mathbb R^W$ the corresponding weight space. It follows that if a ReLU network architecture is capable of expressing any function from $F_{d,n}$ with error $\epsilon$, then, under the hypothesis of continuous weight selection,  the network must have at least $c\epsilon^{-d/n}$ weights. The number of connections is then lower bounded by $\frac{c}{2}\epsilon^{-d/n}$ (since the number of weights is not larger than the sum of the number of computation units and the number of connections, and there are at least as many  connections as units).

The hypothesis of continuous weight selection is crucial in Theorem \ref{th:dwidth}. By examining our proof of the counterpart upper bound $O(\epsilon^{-d/n}\ln(1/\epsilon))$ in Theorem \ref{th:gensmooth}, the weights are selected there in a continuous manner, so this upper bound asymptotically lies above $c\epsilon^{-d/n}$ in agreement with Theorem \ref{th:dwidth}.  We remark, however, that the optimal choice of the network weights (minimizing the error) is known to be discontinuous in general, even for shallow networks (\cite{kainen1999approximation}).

We also compare the bounds of Theorems \ref{th:dwidth} and \ref{th:faster}. In the case $d=n=1$, Theorem \ref{th:dwidth} provides a lower bound $\frac{c}{\epsilon}$ for the number of weights and connections.  On the other hand, in the adaptive architecture scenario, Theorem \ref{th:faster} provides  the upper bound $\frac{c}{\epsilon\ln(1/\epsilon)}$ for the number of weights, connections and computation units. The fact that this latter bound is asymptotically below the bound of Theorem \ref{th:dwidth} reflects the extra expressiveness associated with variable network architecture.   

\subsection{Bounds based on VC-dimension}\label{sec:vc}
In this section we consider the setup where the same network architecture is used to approximate all functions $f\in F_{d,n}$, but the dependence of the weights on $f$ is not assumed to be necessarily continuous. In this setup, some lower bounds on the network complexity can be obtained as a consequence of existing upper bounds on VC-dimension of networks with piece-wise polynomial activation functions and Boolean outputs (\cite{anthony2009neural}). In the next theorem, part a) is a more general but weaker bound, while part b) is a stronger bound assuming a constrained growth of the network depth.

\begin{theor}\label{th:vc}
Fix $d, n$. 
\begin{itemize}
\item[a)] For any $\epsilon\in(0,1)$, a ReLU network architecture capable of approximating any function $f\in F_{d,n}$ with error $\epsilon$ must have at least $c\epsilon^{-d/(2n)}$ weights, with some constant $c=c(d,n)>0$.
\item[b)] Let $p\ge 0, c_1>0$ be some constants. For any $\epsilon\in(0,\frac{1}{2})$, if a ReLU network architecture of depth $L\le c_1\ln^p(1/\epsilon)$ is capable of approximating any function $f\in F_{d,n}$ with error $\epsilon$, then the network must have at least $c_2\epsilon^{-d/n}\ln^{-2p-1} (1/\epsilon)$ weights, with some constant $c_2=c_2(d,n,p,c_1)>0$.\footnote{The author thanks Matus Telgarsky for suggesting this part of the theorem.} 
\end{itemize} 
\end{theor}
\begin{proof} Recall that given a class $H$ of Boolean functions on $[0,1]^d$, the VC-dimension of $H$ is defined as the  size of the largest shattered subset $S\subset [0,1]^d$, i.e. the largest subset on which $H$ can compute any dichotomy (see, e.g., \cite{anthony2009neural}, Section 3.3). We are interested in the case when $H$ is the family of functions obtained by applying thresholds $\mathbbm{1}(x>a)$ to a ReLU network with fixed architecture but variable weights. In this case Theorem 8.7 of \cite{anthony2009neural} implies that 
\begin{equation}\label{eq:vc1}
\operatorname{VCdim}(H)\le c_3W^2,\end{equation}
and Theorem 8.8 implies that 
\begin{equation}\label{eq:vc2}\operatorname{VCdim}(H)\le c_3L^2W\ln W,
\end{equation} 
where $W$ is the number of weights, $L$ is the network depth, and $c_3$ is an absolute constant.

Given a positive integer $N$ to be chosen later, choose $S$ as a set of $N^d$ points $\mathbf x_1,\ldots, \mathbf x_{N^d}$ in the cube $[0,1]^d$ such that the distance between any two of them is not less than $\frac{1}{N}$. Given any assignment of values $y_1,\ldots,y_{N^d}\in \mathbb R$, we can construct a smooth function $f$ satisfying $f(\mathbf x_m)=y_m$ for all $m$ by setting
\begin{equation}\label{fy}
f(\mathbf x) = \sum_{m=1}^{N^d}y_m\phi(N(\mathbf x-\mathbf x_m)),
\end{equation}
with some $C^\infty$ function $\phi:\mathbb R^d\to\mathbb R$ such that $\phi(\mathbf 0)=1$ and $\phi(\mathbf x)=0$ if $|\mathbf x|>\frac{1}{2}$. 

Let us obtain a condition ensuring that such $f\in F_{d,n}$. For any multi-index $\mathbf{n}$, $$\max_{\mathbf x}|D^{\mathbf{n}}f(\mathbf x)|= N^{|\mathbf n|}\max_{m}|y_m|\max_{\mathbf x} |D^{\mathbf n} \phi(\mathbf x)|,$$
so if 
\begin{equation}\label{maxy}
\max_m|y_m|\le c_4 N^{-n},
\end{equation}
with the constant $c_4=(\max_{\mathbf n:|\mathbf n|\le n}\max_{\mathbf x} |D^{\mathbf n} \phi(\mathbf x)|)^{-1}$, then $f\in F_{d,n}$.

Now set 
\begin{equation}\label{eps}
\epsilon=\frac{c_4}{3}N^{-n}.
\end{equation} Suppose that there is a ReLU network architecture $\eta$ that can approximate, by adjusting its weights, any $f\in F_{d,n}$ with error less than $\epsilon$. Denote by $\eta(\mathbf x, \mathbf w)$ the output of the network for the input vector $\mathbf x$ and the vector of weights $\mathbf w$. 

Consider any assignment $\mathbf z$ of Boolean values $z_1, \ldots, z_{N^d}\in\{0,1\}$. Set
$$y_m=z_m c_4 N^{-n},\quad m=1,\ldots,N^d,$$
and let $f$ be given by \eqref{fy} (see Fig. \ref{fig:th2}); then \eqref{maxy} holds and hence $f\in F_{d,n}$. 
\begin{figure}
\begin{center}
        \includegraphics[width=0.5\textwidth, trim={145mm 65mm 165mm 115mm},clip]{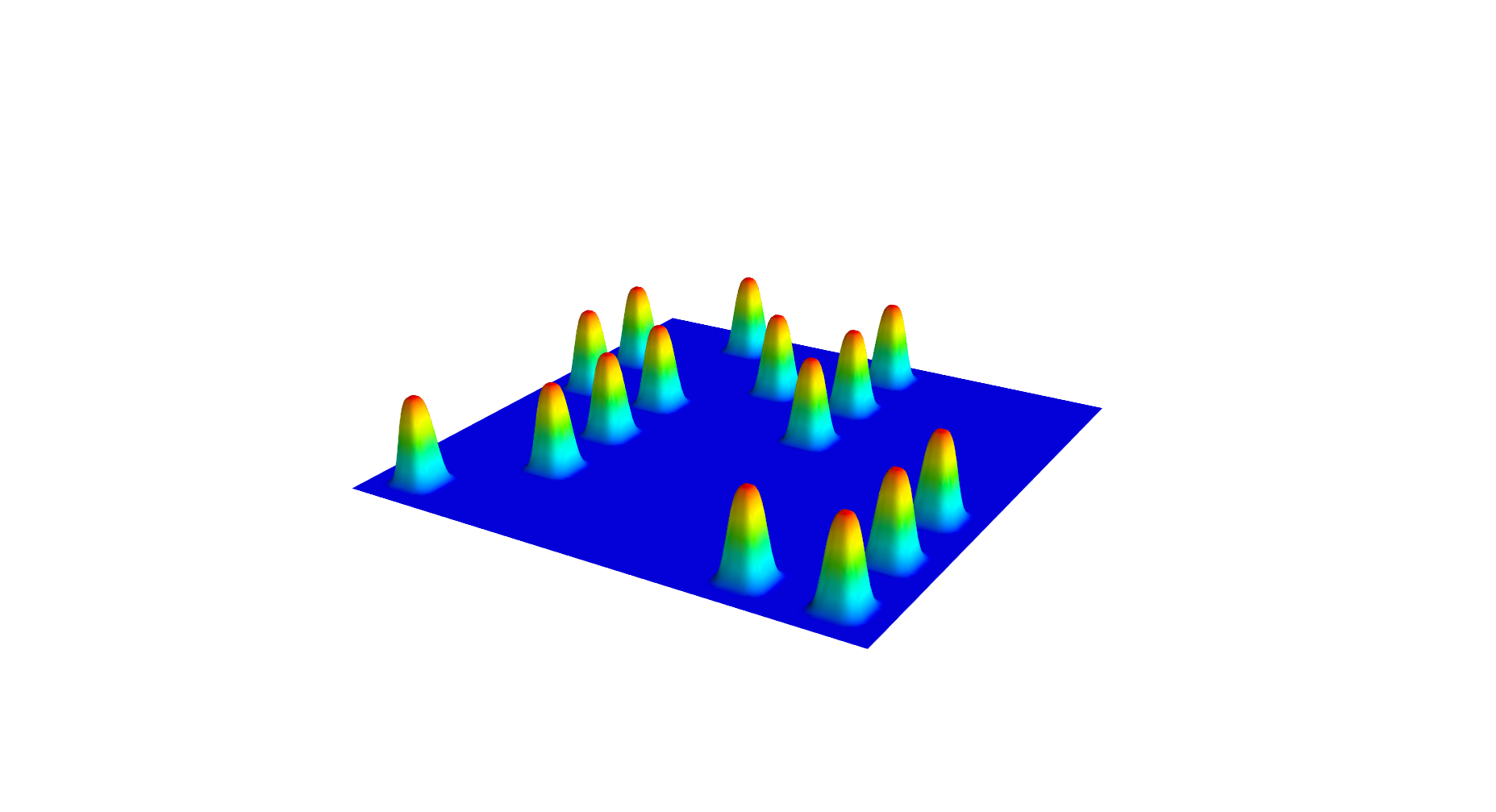}
\caption{A function $f$ considered in the proof of Theorem 2 (for $d=2$).}
\label{fig:th2}
\end{center}
\end{figure}
By assumption, there is then a vector of weights, $\mathbf w=\mathbf w_{\mathbf z}$, such that for all $m$ we have $|\eta(\mathbf x_m, \mathbf w_{\mathbf z})-y_m|\le\epsilon,$ and in particular
$$\eta(\mathbf x_m, \mathbf w_{\mathbf z}) 
\left\{\arraycolsep=1.5pt
\begin{array}{cccl}
\ge & c_4N^{-n}-\epsilon & > c_4N^{-n}/2, & \quad\text{ if } z_m = 1,\\
\le & \epsilon & < c_4N^{-n}/2, & \quad\text{ if } z_m = 0,
\end{array}
\right.
$$ so the thresholded network $\eta_1=\mathbbm{1}(\eta > c_4N^{-n}/2)$ has outputs $$\eta_1(\mathbf x_m, \mathbf w_{\mathbf z})=z_m, \quad m=1,\ldots, N^d.$$
Since the Boolean values $z_m$ were arbitrary, we conclude that the subset $S$ is shattered and hence 
\begin{equation*}
\operatorname{VCdim}(\eta_1)\ge N^d.
\end{equation*}
Expressing $N$ through $\epsilon$ with \eqref{eps}, we obtain
\begin{equation}\label{vcdim1}
\operatorname{VCdim}(\eta_1)\ge \Big(\frac{3\epsilon}{c_4}\Big)^{-d/n}.
\end{equation}
To establish part a) of the Theorem, we apply bound \eqref{eq:vc1} to the network $\eta_1$:
\begin{equation}\label{vcdim2}
\operatorname{VCdim}(\eta_1)\le c_3W^2,
\end{equation}
where $W$ is the number of weights in $\eta_1$, which is the same as in $\eta$ if we do not count the threshold parameter.  Combining \eqref{vcdim1} with \eqref{vcdim2}, we obtain the desired lower bound $W\ge c\epsilon^{-d/(2n)}$ with $c=(c_4/3)^{d/(2n)}c_3^{-1/2}$.

To establish part b) of the Theorem, we use bound \eqref{eq:vc2} and the hypothesis $L\le c_1\ln^p(1/\epsilon)$:
\begin{equation}\label{vcdim3}
\operatorname{VCdim}(\eta_1) \le c_3c_1^2\ln^{2p}(1/\epsilon) W\ln W.
\end{equation}
Combining \eqref{vcdim1} with \eqref{vcdim3}, we obtain
\begin{equation}\label{vcdim4}
W\ln W\ge \frac{1}{c_3c_1^2}\Big(\frac{3\epsilon}{c_4}\Big)^{-d/n}\ln^{-2p}(1/\epsilon).
\end{equation}
Trying a $W$ of the form $W_{c_2}=c_2\epsilon^{-d/n}\ln^{-2p-1} (1/\epsilon)$ with a constant $c_2$, we get 
\begin{align*}
W_{c_2}\ln W_{c_2} &=c_2\epsilon^{-d/n}\ln^{-2p-1} (1/\epsilon)\Big(\frac{d}{n}\ln(1/\epsilon)+\ln c_2-(2p+1)\ln\ln(1/\epsilon)\Big)\\
&=\Big(c_2\frac{d}{n}+o(1)\Big)\epsilon^{-d/n}\ln^{-2p}(1/\epsilon).\end{align*}
Comparing this with \eqref{vcdim4}, we see that if we choose $c_2<(c_4/3)^{d/n}n/(dc_3c_1^2)$, then for sufficiently small $\epsilon$ we have $W\ln W\ge W_{c_2}\ln W_{c_2}$ and hence $W\ge W_{c_2}$, as claimed. We can ensure that $W\ge W_{c_2}$ for all $\epsilon\in(0,\frac{1}{2})$ by further decreasing $c_2$.
\end{proof}

We remark that the constrained depth hypothesis of part b) is satisfied, with $p=1$, by the architecture used for the upper bound in Theorem \ref{th:gensmooth}. The bound stated in part b) of Theorem \ref{th:vc} matches the upper bound of Theorem \ref{th:gensmooth} and the lower bound of Theorem \ref{th:dwidth} up to a power of $\ln(1/\epsilon)$.

\subsection{Adaptive network architectures}\label{sec:lb_funcdep}

Our goal in this section is to obtain a lower bound for the approximation complexity in the scenario where the network architecture may depend on the approximated function. This lower bound is thus a counterpart to the upper bound of Section \ref{sec:faster}. 

To state this result we define the complexity $\mathcal N(f,\epsilon)$ of approximating the function $f$ with error $\epsilon$ as the minimal number of hidden computation units in a ReLU network that provides such an approximation. 

\begin{theor}\label{th:lb_funcdep}
For any $d,n$, there exists $f\in \mathcal W^{n,\infty}([0,1]^d)$ such that $\mathcal N(f,\epsilon)$ is not $o(\epsilon^{-d/(9n)})$ as $\epsilon\to 0$.
\end{theor}
The proof relies on the following lemma. 
\begin{lemma}\label{th:embed}
Fix $d,n$. For any sufficiently small $\epsilon>0$ there exists $f_\epsilon\in F_{d,n}$ such that $\mathcal N(f_\epsilon,\epsilon)\ge c_1\epsilon^{-d/(8n)}$, with some constant $c_1=c_1(d,n)>0$.
\end{lemma}
\begin{proof} Observe that all the networks with not more than $m$ hidden computation units can be embedded in the single ``enveloping'' network that has $m$ hidden layers, each consisting of $m$ units, and that includes all the connections between units not in the same layer (see Fig. \ref{fig:envelope}). The number of weights in this enveloping network is $O(m^4)$. On the other hand, Theorem \ref{th:vc}a) states that at least $c\epsilon^{-d/(2n)}$ weights are needed for an architecture capable of $\epsilon$-approximating any function in $F_{d,n}$. It follows that there is a function $f_\epsilon\in F_{d,n}$ that cannot be $\epsilon$-approximated by networks with fewer than $c_1\epsilon^{-d/(8n)}$ computation units. 
\end{proof}

\begin{figure}
\begin{center}
    \begin{subfigure}[b]{0.45\textwidth}
        \includegraphics[width=\textwidth, trim={25mm 20mm 44mm 20mm},clip]{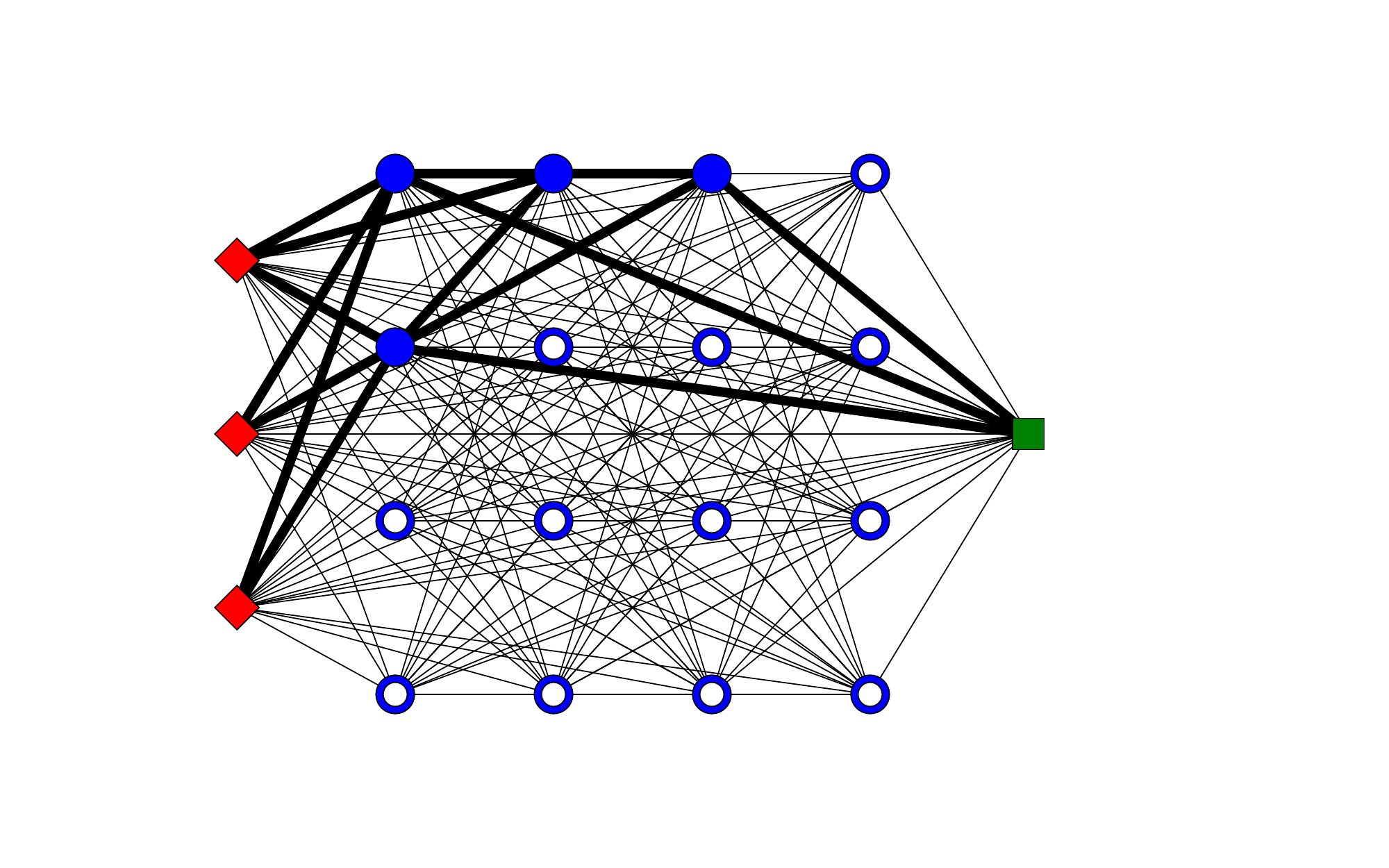}
        \caption{}
        \label{fig:envelope}
    \end{subfigure}    
    \begin{subfigure}[b]{0.35\textwidth}
        \includegraphics[width=\textwidth, trim={20mm 5mm 20mm 5mm},clip]{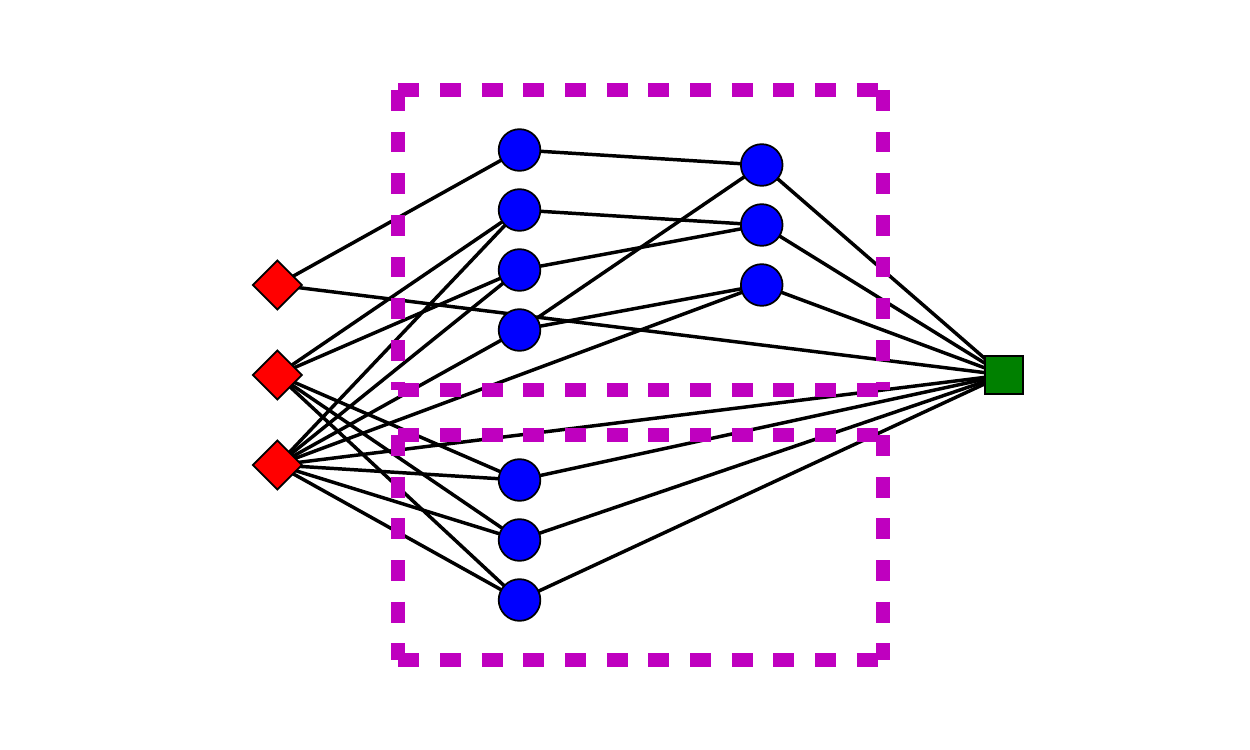}
        \caption{}
        \label{fig:paired_net}
    \end{subfigure}
\caption{(a) Embedding a network with $m=4$ hidden units into an ``enveloping'' network (see Lemma \ref{th:embed}). (b) Putting sub-networks in parallel to form an approximation for the sum or difference of two functions, see Eq. \eqref{Nf1f2}.}
\label{fig:lemma1}
\end{center}
\end{figure}

Before proceeding to the proof of Theorem \ref{th:lb_funcdep}, note that $\mathcal N(f,\epsilon)$ is a monotone decreasing function of $\epsilon$ with a few obvious properties:  
\begin{equation}\label{Naf}
\mathcal N(af, |a|\epsilon) = \mathcal N(f, \epsilon), \text{ for any } a\in\mathbb R\setminus\{0\}
\end{equation}
(follows by multiplying the weights of the output unit of the approximating network by a constant);
\begin{equation}\label{Nfg}
\mathcal N(f\pm g, \epsilon+\|g\|_\infty) \le \mathcal N(f, \epsilon)
\end{equation}
(follows by approximating $f\pm g$ by an approximation of $f$);
\begin{equation}\label{Nf1f2}
\mathcal N(f_1\pm f_2,\epsilon_1+\epsilon_2)\le \mathcal N(f_1,\epsilon_1)+\mathcal N(f_2, \epsilon_2)
\end{equation}
(follows by combining approximating networks for $f_1$ and $f_2$ as in Fig. \ref{fig:paired_net}). 

\begin{proof}[Proof of Theorem \ref{th:lb_funcdep}.] The claim of Theorem \ref{th:lb_funcdep} is similar to the claim of Lemma \ref{th:embed}, but is about a single function $f$ satisfying a slightly weaker complexity bound at multiple values of $\epsilon\to 0$. We will assume that Theorem \ref{th:lb_funcdep} is false, i.e., \begin{equation}\label{assum}
\mathcal N(f,\epsilon)=o(\epsilon^{-d/(9n)})
\end{equation} for all $f\in \mathcal W^{n,\infty}([0,1]^d),$ and we will reach contradiction by presenting $f$ violating this assumption. Specifically, we construct this $f$ as  
\begin{equation}\label{fser}
f=\sum_{k=1}^\infty a_k f_k,
\end{equation}
with some $a_k\in \mathbb R$, $f_k\in F_{d,n}$, and we will make sure that 
\begin{equation}\label{nfek}
\mathcal N(f, \epsilon_k)\ge \epsilon_k^{-d/(9n)}
\end{equation}
for a sequence of $\epsilon_k\to 0$. 

We determine $a_k, f_k, \epsilon_k$ sequentially. Suppose we have already found $\{a_s,f_s,\epsilon_s\}_{s=1}^{k-1}$; let us describe how we define $a_k, f_k, \epsilon_k$.

First, we set 
\begin{equation}\label{ak}
a_k=\min_{s=1,\ldots,k-1}\frac{\epsilon_s}{2^{k-s}}.
\end{equation}
In particular, this ensures that 
\begin{equation*}
a_k\le \epsilon_1 2^{1-k},\end{equation*}
so that the function $f$ defined by the series \eqref{fser} will be in $\mathcal W^{n,\infty}([0,1]^d)$, because $\|f_k\|_{\mathcal W^{n,\infty}([0,1]^d)}\le 1$. 

Next, using Lemma \ref{th:embed} and Eq. \eqref{Naf}, observe that if $\epsilon_k$ is sufficiently small, then we can find $f_k\in F_{d,n}$ such that
\begin{equation}\label{ekfk}
\mathcal N\big(a_kf_k, 3\epsilon_k\big)=\mathcal N\Big(f_k, \frac{3\epsilon_k}{a_k}\Big)\ge c_1\Big(\frac{3\epsilon_k}{a_k}\Big)^{-d/(8n)}\ge 2 \epsilon_k^{-d/(9n)}.
\end{equation}
In addition, by assumption \eqref{assum}, if $\epsilon_k$ is small enough then
\begin{equation}\label{ek}
\mathcal N\Big(\sum_{s=1}^{k-1}a_sf_s, \epsilon_k\Big) \le \epsilon_k^{-d/(9n)}.
\end{equation}
Let us choose $\epsilon_k$ and $f_k$ so that both \eqref{ekfk} and \eqref{ek} hold. Obviously, we can also make sure that $\epsilon_k\to 0$ as $k\to\infty$.

Let us check that the above choice of $\{a_k, f_k,\epsilon_k\}_{k=1}^\infty$ ensures that inequality \eqref{nfek} holds for all $k$: 
\begin{align*}
\mathcal N\Big(\sum_{s=1}^{\infty}a_sf_s, \epsilon_k\Big)
&\ge \mathcal N\Big(\sum_{s=1}^{k}a_sf_s, \epsilon_k+\Big\|\sum_{s=k+1}^\infty a_sf_s\Big\|_\infty\Big)\\
&\ge \mathcal N\Big(\sum_{s=1}^{k}a_sf_s, \epsilon_k+\sum_{s=k+1}^\infty a_s\Big)\\
&\ge \mathcal N\Big(\sum_{s=1}^{k}a_sf_s, 2\epsilon_k\Big)\\
&\ge \mathcal N(a_kf_k, 3\epsilon_k)-\mathcal N\Big(\sum_{s=1}^{k-1}a_sf_s, \epsilon_k\Big)\\
&\ge \epsilon^{-d/(9n)}.
\end{align*}
Here in the first step we use inequality \eqref{Nfg}, in the second the monotonicity of $\mathcal N(f,\epsilon)$, in the third the monotonicity of $\mathcal N(f,\epsilon)$ and the setting \eqref{ak}, in the fourth the inequality \eqref{Nf1f2}, and in the fifth the conditions \eqref{ekfk} and \eqref{ek}.
\end{proof}

\subsection{Slow approximation of smooth functions by shallow networks}\label{sec:slow}
In this section we show that, in contrast to deep ReLU networks, shallow ReLU networks relatively inefficiently approximate sufficiently smooth ($C^2$) nonlinear functions. We remark that Liang and Srikant \citeyear{liang2016why} prove a similar result assuming global convexity instead of smoothness and nonlinearity. 

\begin{theor}\label{th:slow} Let $f\in C^2([0,1]^d)$ be a nonlinear function (i.e., not of the form $f(x_1,\ldots,x_d)\equiv a_0+\sum_{k=1}^d a_k x_k$ on the whole $[0,1]^d$). Then, for any fixed $L$, a depth-$L$ ReLU network approximating $f$ with error $\epsilon\in(0,1)$ must have at least $c\epsilon^{-1/(2(L-2))}$ weights and computation units, with some constant $c=c(f,L)>0.$
\end{theor}
\begin{proof}
Since $f\in C^2([0,1]^d$ and $f$ is nonlinear, we can find $\mathbf x_0\in [0,1]^d$ and $\mathbf v\in\mathbb R^d$ such that $\mathbf x_0+x\mathbf v\in[0,1]^d$ for all $x\in [-1,1]$ and the function $f_1:x\mapsto f(\mathbf x_0+x\mathbf v)$ is strictly convex or concave on $[-1,1]$. Suppose without loss of generality that it is strictly convex: 
\begin{equation}\label{eq:d2}
\min_{x\in [-1,1]} f_1''(x)=c_1>0.
\end{equation}
Suppose that $\widetilde f$ is an $\epsilon$-approximation of function $f$, and let $\widetilde f$ be implemented by a ReLU network $\eta$ of depth $L$. Let $\widetilde f_1:x\mapsto \widetilde f(\mathbf x_0+x\mathbf v)$. Then $\widetilde f_1$ also approximates $f_1$ with error not larger than $\epsilon$. Moreover, since $\widetilde f_1$ is obtained from $\widetilde f$ by a linear substitution $\mathbf x=\mathbf x_0+x\mathbf v$, $\widetilde f_1$ can be implemented by a ReLU network $\eta_1$ of the same depth $L$ and with the number of units and weights not larger than in $\eta$ (we can obtain $\eta_1$ from $\eta$ by replacing the input layer in $\eta$ with a single unit, accordingly modifying the input connections, and adjusting the weights associated with these connections). It is thus sufficient to establish the claimed bounds for $\eta_1$.

By construction, $\widetilde f_1$ is a continuous piece-wise linear function of $x$. Denote by $M$ the number of linear pieces in $\widetilde f_1$. We will use the following counting lemma. 
\begin{lemma}\label{th:count} $M\le (2U)^{L-2}$, where $U$ is the number of computation units in $\eta_1$.
\end{lemma}
\begin{proof} This bound, up to minor details, is proved in Lemma 2.1 of \cite{telgarsky2015representation}. Precisely, Telgarsky's lemma states that if a network has a single input, connections only between neighboring layers, at most $m$ units in a layer, and a piece-wise linear activation function with $t$ pieces, then the number of linear pieces in the output of the network is not greater than $(tm)^L$. By examining the proof of the lemma we see that it will remain valid for networks with connections not necessarily between neighboring layers, if we replace $m$ by $U$ in the expression $(tm)^L$. Moreover, we can slightly strengthen the bound by noting that in the present paper the input and output units are counted as separate layers,  only units of layers 3 to $L$ have multiple incoming connections, and the activation function is applied only in layers 2 to $L-1$. By following Telgarsky's arguments, this gives the slightly more accurate bound $(tU)^{L-2}$ (i.e., with the power $L-2$ instead of $L$). It remains to note that the ReLU activation function corresponds to $t=2$.   
\end{proof}

Lemma \ref{th:count} implies that there is an interval $[a,b]\subset [-1,1]$ of length not less than $2(2U)^{-(L-2)}$ on which the function $\widetilde f_1$ is linear. Let $g=f_1-\widetilde f_1.$ Then, by the approximation accuracy assumption, $\sup_{x\in[a,b]}|g(x)|\le \epsilon$, while by \eqref{eq:d2} and by the linearity of $\widetilde f_1$ on $[a,b]$, $\max_{x\in [a,b]} g''(x)\ge c_1>0.$ It follows that $\max(g(a),g(b))\ge g(\frac{a+b}{2})+\frac{c_1}{2}(\frac{b-a}{2})^2$ and hence 
$$\epsilon\ge \frac{1}{2}\big(\max(g(a),g(b))-g(\tfrac{a+b}{2})\big)\ge\frac{c_1}{4}\Big(\frac{b-a}{2}\Big)^2\ge \frac{c_1}{4}(2U)^{-2(L-2)}, $$
which implies the claimed bound $U\ge \frac{1}{2}(\frac{4\epsilon}{c_1})^{-1/(2(L-2))}$. Since there are at least as many weights as computation units in a network, a similar bound holds for the number of weights.
\end{proof}

\section{Discussion}\label{sec:discus}
We discuss some implications of the obtained bounds.
\paragraph{Deep vs. shallow ReLU approximations of smooth functions.} Our results clearly show that deep ReLU networks more efficiently express smooth functions than shallow ReLU networks. By Theorem \ref{th:gensmooth}, functions from the Sobolev space $\mathcal W^{n,\infty}([0,1]^d)$ can be $\epsilon$-approximated by ReLU networks with depth $O(\ln(1/\epsilon))$ and the number of computation units $O(\epsilon^{-d/n}\ln(1/\epsilon))$. In contrast, by Theorem \ref{th:slow}, a nonlinear function from $C^2([0,1]^d)$ cannot be $\epsilon$-approximated by a ReLU network of fixed depth $L$ with the number of units less than $c\epsilon^{-1/(2(L-2))}.$ In particular, it follows that in terms of the required number of computation units, unbounded-depth approximations of functions from $\mathcal W^{n,\infty}([0,1]^d)$ are asymptotically strictly more efficient than approximations with a fixed depth $L$ at least when $$\frac{d}{n}<\frac{1}{2(L-2)}$$ (assuming also $n>2$, so that $\mathcal W^{n,\infty}([0,1]^d)\subset C^2([0,1]^d)$).  
The efficiency of depth is even more pronounced for very smooth functions such as polynomials, which can be implemented by deep networks using only $O(\ln(1/\epsilon))$ units (cf. Propositions \ref{th:x2} and \ref{th:fg} and the proof of Theorem \ref{th:gensmooth}).  Liang and Srikant describe in \cite{liang2016why} some conditions on the approximated function (resembling conditions of local analyticity) under which complexity of deep $\epsilon$-approximation is $O(\ln^c(1/\epsilon))$ with a constant power $c$.  
\paragraph{Continuous model selection vs. function-dependent network architectures.} When approximating a function by a neural network, one can either view the network architecture as fixed and only tune the weights, or optimize the architecture as well. Moreover, when tuning the weights, one can either require them to continuously depend on the approximated function or not. We naturally expect that more freedom in the choice of the approximation should lead to higher expressiveness. 

Our bounds confirm this expectation to a certain extent. Specifically, the complexity of $\epsilon$-approximation of functions from the unit ball $F_{1,1}$ in $\mathcal W^{1,\infty}([0,1])$ is lower bounded by $\frac{c}{\epsilon}$ in the scenario with a fixed architecture and continuously selected weights (see Theorem \ref{th:dwidth}). On the other hand, we show in Theorem \ref{th:faster} that this complexity is upper bounded by $O(\frac{1}{\epsilon\ln(1/\epsilon)})$ if we are allowed to adjust the network architecture. This bound is achieved by finite-depth (depth-6) ReLU networks using the idea of reused subnetworks familiar from the theory of Boolean circuits \cite{shannon1949synthesis}. 

In the case of fixed architecture, we have not established any evidence of complexity improvement for unconstrained weight selection compared to continuous weight selection. We remark however that, already for approximations with depth-3 networks, the optimal weights are known to  discontinuously depend, in general, on the approximated function (\cite{kainen1999approximation}). On the other hand, part b) of Theorem \ref{th:vc} shows that if the network depth scales as $O(\ln^p(1/\epsilon))$, discontinuous weight selection cannot improve the continuous-case complexity more than by a factor being some power of $\ln(1/\epsilon)$.
\paragraph{Upper vs. lower complexity bounds.} We indicate the gaps between respective upper and lower bounds in the three scenarios mentioned above: fixed architectures with continuous selection of weights, fixed architectures with unconstrained selection of weights, or adaptive architectures. 

For fixed architectures with continuous selection the lower bound $c\epsilon^{-d/n}$ is provided by Proposition \ref{th:dwidth}, and the upper bound $O(\epsilon^{-d/n}\ln(1/\epsilon))$ by Theorem \ref{th:gensmooth}, so these bounds are tight up to a factor $O(\ln(1/\epsilon))$. 

In the case of fixed architecture but unconstrained selection, part b) of Theorem \ref{th:vc} gives a lower bound $c\epsilon^{-d/n}\ln^{-2p-1} (1/\epsilon)$ under assumption that the depth is constrained by $O(\ln^p(1/\epsilon))$. This is only different by a factor of $O(\ln^{2p+2}(1/\epsilon))$ from the upper bound of Theorem \ref{th:gensmooth}. Without this depth constraint we only have the significantly weaker bound $c\epsilon^{-d/(2n)}$ (part a) of Theorem \ref{th:vc}). 

In the case of adaptive architectures, there is a big gap between our upper and lower bounds. The upper bound $O(\frac{1}{\epsilon\ln(1/\epsilon)})$ is given by Theorem \ref{th:faster} for $d=n=1$. The lower bound, proved for general $d,n$ in Theorem \ref{th:lb_funcdep}, only states that there are $f\in\mathcal W^{n,\infty}([0,1]^d)$ for which the complexity is not $o(\epsilon^{-d/(9n)})$.

\paragraph{ReLU vs. smooth activation functions.} A popular general-purpose method of approximation is shallow (depth-3) networks with smooth activation functions (e.g., logistic sigmoid). Upper and lower approximation complexity bounds for  these networks (\cite{mhaskar1996neural, maiorov2000near}) show that complexity scales as $\sim \epsilon^{-d/n}$ up to some $\ln(1/\epsilon)$ factors. Comparing this with our bounds in Theorems \ref{th:gensmooth},\ref{th:faster},\ref{th:vc}, it appears that deep ReLU networks are roughly (up to $\ln(1/\epsilon)$ factors) as expressive as shallow networks with smooth activation functions. 
 
\paragraph{Conclusion.} We have established several upper and lower bounds for the expressive power of deep ReLU networks in the context of approximation in Sobolev spaces. We should note, however, that this setting may not quite reflect typical real world applications, which usually possess symmetries and hierarchical and other structural properties substantially narrowing the actually interesting classes of approximated functions (\cite{lecun2015deep}). Some recent publications introduce and study expressive power of deep networks in frameworks bridging this gap, in particular, graph-based hierarchical approximations are studied in \cite{mhaskar2016learning, mhaskar2016deep} and convolutional arithmetic circuits in \cite{cohen2015expressive}. Theoretical analysis of expressiveness of deep networks taking into account such properties of real data seems to be an important and promising direction of future research.

\section*{Acknowledgments} The author thanks Matus Telgarsky and the anonymous referees for multiple helpful comments on the preliminary versions of the paper. The research was funded by Skolkovo Institute of Science and Technology.

\bibliographystyle{plain}
\bibliography{relu}

\end{document}